\documentclass[final,5p,times,twocolumn]{elsarticle}

\usepackage{amssymb}
\usepackage{epstopdf}
\usepackage[english]{babel}
 \usepackage{lineno}
\usepackage{subfigure}

\usepackage{graphicx}
\usepackage{url}
\usepackage{pslatex}
\usepackage{bm}
\usepackage{natbib}
\usepackage{amsmath,amssymb,amsthm}  %~\\¿ÉÓÃÓÚÔÚ¶¨Àí¼°Ö¤Ã÷µÈ»µ¾³ÏÂµÄ»»ÐÐ
\usepackage{paralist}%[flushleft] Ê¹ÁÐ±í×ó¶ÔÆëµÄÃüÁî
\usepackage{color}
\usepackage{epstopdf}
\journal{Signal Processing}
\usepackage{color}
\usepackage{subfigure}

\usepackage{graphicx}
\usepackage{url}
\usepackage{pslatex}
\usepackage{bm}
\usepackage{amsmath,amssymb,amsthm}  %~\\可用于在定理及证明等坏境下的换行
\usepackage{paralist}%[flushleft] 使列表左对齐的命令
\usepackage{algorithm}
\usepackage{algorithmic}        %用到的宏包，要自己改下
\usepackage{multirow}
\renewcommand{\algorithmicrequire}{\textbf{Initialization:}}   %改成后面的小标题
\renewcommand{\algorithmicensure}{\textbf{Iteration:}}
\renewcommand{\algorithmiclastcon}{\textbf{Output:}} %Used for the algorithm list

\newtheorem{lemma}{Lemma}

\newtheorem{theorem}{Theorem}

\newtheorem{prop}{Proposition}
\newcommand{\e}{\begin{equation}}
\newcommand{\ee}{\end{equation}}
\newcommand{\en}{\begin{equation*}}
\newcommand{\een}{\end{equation*}}
\newcommand{\eqn}{\begin{eqnarray}}
\newcommand{\eeqn}{\end{eqnarray}}
\newcommand{\bmat}{\begin{bmatrix}}
\newcommand{\emat}{\end{bmatrix}}
\newcommand{\BIT}{\begin{itemize}}
\newcommand{\EIT}{\end{itemize}}
\newcommand{\ve}{\bm e}

\newcommand{\vx}{\bm x}
\newcommand{\vy}{\bm y}

\newcommand{\vtheta}{\bm \theta}

\newcommand{\mA}{\bm A}
\newcommand{\mB}{\bm B}

\newcommand{\mI}{\bm I}

\newcommand{\mPhi}{\bm \Phi}
\newcommand{\mPsi}{\bm \Psi}

\newcounter{oursection}

\begin{document}

\begin{frontmatter}



\title{Online Learning Sensing Matrix and Sparsifying Dictionary Simultaneously\\ for Compressive Sensing}


\author[TH]{Tao Hong}
\ead{hongtao@cs.technion.ac.il}
\author[ZZ]{Zhihui Zhu}
\ead{zzhu29@jhu.edu}
	%\author{Michael~Shell,~\IEEEmembership{Member,~IEEE,}
	%        John~Doe,~\IEEEmembership{Fellow,~OSA,}
	%        and~Jane~Doe,~\IEEEmembership{Life~Fellow,~IEEE}% <-this % stops a space
	% % revised December 27, 2012.
%	\thanks{This research is supported in part by ERC Grant agreement no. 320649, and in part by the Intel Collaborative Research Institute for Computational Intelligence (ICRI-CI). The code
%		in this paper to represent the experiments can be downloaded through the link {https://github.com/happyhongt/}.}%revised December 27, 2012.

	\address[TH]{Department
		of Computer Science, Technion - Israel Institute of Technology, Haifa, 32000, Israel.}%mzib@cs.technion.ac.il
	%\thanks{Y. Wang is with the Department of Electrical Engineering, Shanghai Jiao Tong University, 200240, China (e-mail: Queen\_Wang@sjtu.edu.cn).}
	\address[ZZ]{Center for Imaging Science,
		Johns Hopkins University, Baltimore, MD 21218, USA} %GA, 30332 USA e-mail: (see http://www.michaelshell.org/contact.html).}% <-this % stops a space
 %  \address{}

%\address{}

\begin{abstract}
This paper considers the problem of simultaneously learning the Sensing Matrix and Sparsifying Dictionary (SMSD)  on a large training dataset. To address the formulated joint learning problem, we propose an online algorithm that consists of a closed-form solution for optimizing the sensing matrix with a fixed sparsifying dictionary and a stochastic method for learning the sparsifying dictionary on a large dataset when the sensing matrix is given. Benefiting from training on a large dataset, the obtained compressive sensing (CS) system by the proposed algorithm yields a much better performance in terms of signal recovery accuracy than the existing ones. The simulation results on natural images demonstrate the effectiveness of the suggested online algorithm compared with the existing methods.
\end{abstract}


\end{frontmatter}

\section{Introduction}\label{S_1}
Sparse representation (Sparseland) has led to numerous successful applications spanning through many fields, including image processing, machine learning, pattern recognition, and compressive sensing (CS) \cite{M08} - \cite{CW08}. This model assumes that a signal $\bm x\in \Re^{N}$ can be represented as a linear combination of a few columns, also known as atoms, taken from a matrix $\bm\Psi\in\Re^{N\times L}$ (referred to as a dictionary):
\e
\bm x=\bm\Psi\bm \theta+\bm e,\label{sparse_representation}
\ee
where $\bm\theta\in\Re^{L}$ has few non-zero entries and is the representation coefficient vector of $\bm x$ over the dictionary $\bm \Psi$ and $\bm e\in\Re^N$ is known as the sparse representation error (SRE) which is not nil in general case. The signal $\bm x$ is called  $K$-sparse in $\bm\Psi$ if $\|\bm\theta\|_0\leq K$ where $\|\bm\theta\|_0$ is used to count the number of non-zeros in $\bm\theta$.% and named $\ell_0$ norm even it is not a true norm.

The choice of dictionary $\bm\Psi$ depends on specific applications and can be a predefined one, e.g., discrete cosine transform (DCT), wavelet transform and a multiband modulated discrete
prolate spheroidal sequences (DPSS's) dictionary \cite{ZhuWakin2015MDPSS} etc.  It is also beneficial and recently widely-utilized to adaptively learn a dictionary  $\bm \Psi$, called dictionary learning, such that a set of $P$ training signals $\{\bm x_k,k=1,2,\cdots,P\}$ is sparsely represented by optimizing a $\bm\Psi$. There exist many efficient algorithms to learn a dictionary \cite{TF11} and the most two popular methods among them are the method of optimal directions (MOD) \cite{EAH99} and the K-singular value decomposition (KSVD) algorithm \cite{AEB06}. In particular, we prefer to use an over-complete dictionary \cite{AEB06}, $N<L$. % and the method for designing incoherent dictionary~\cite{li2017new}.
%In order to make the dictionary satisfy the constraint, each column of the dictionary is normalized to have a unit $\ell_2$ norm directly.

CS is an emerging framework that enables to exactly recover the signal $\bm x$, in which it is sparse or sparsely represented by a dictionary $\bm\Psi$, from a number of linear measurements that is considerably lower than the size of samples required by the Shannon-Nyquist theorem \cite{CW08}. Generally speaking, researchers tend to utilize a random matrix $\bm \Phi\in \Re^{M\times N}$ as the sensing matrix (a.k.a projection matrix) to obtain the linear measurements
\e
\bm y = \bm \Phi \bm x = \bm \Phi \bm \Psi \bm \theta + \bm\Phi\bm e,
\label{eq:y}\ee
where $M\ll N$. Abundant efforts have been devoted to optimize the sensing matrix with a predefined dictionary resulting in a CS system that outperforms the standard one (random matrix) in various cases~\cite{E07} - \cite{BLLLJC15}.
%Typically, the optimized projection matrix is obtained by solving the following problem %\cite{LZYCB13} - \cite{TZ16}.
%\e
%\min_{\bm \Phi} f_1(\bm\Phi) \triangleq \|\bm G_t-\bm\Psi^\mathcal T\bm\Phi^\mathcal T\bm\Phi\bm\Psi\|_F^2\label{Optimized:Projection}
%\ee
%where $\bm G_t$ can be the identity matrix (Tight Frame) or Equiangular Tight Frame (ETF) and $\mathcal T$ denotes the transpose operator.

Recently, researchers realize simultaneously optimizing sensing matrix and dictionary for the CS system yields a higher signal reconstruction accuracy than the classical CS systems which only optimize sensing matrix with a fixed dictionary \cite{DCS09,BLLLJC15}. The main idea underlying in \cite{DCS09,BLLLJC15} is to consider the influence of SRE in learning the dictionary (see Section~\ref{S_3} for the formal problem).   Alternating minimization methods are introduced to jointly design the sensing matrix $\bm\Phi$ and the dictionary $\bm\Psi$ in \cite{DCS09,BLLLJC15}. Compared to \cite{DCS09}, closed-form solutions for updating the sensing matrix and the dictionary are derived in \cite{BLLLJC15} which hence obtains a better performance in terms of signal recovery accuracy. The disadvantage of the method in \cite{BLLLJC15} is that it involves many singular value decompositions (SVDs) making their algorithm inefficient in practice. %sacrifice the efficiency to obtain a accurate solution.

Although the  method for jointly optimizing the sensing matrix and the dictionary in \cite{DCS09,BLLLJC15} works well for a small-scale training dataset (e.g., $N = 64$ and $P=10^4$), it becomes inefficient (and even impractical) if the dimension of the dictionary is high or the size of training dataset is very large (say with more than $10^6$ patches in natural images situation) or for the case involving dynamic data like video stream. It is easy to see that the methods in \cite{DCS09,BLLLJC15} require heavy memory and computations to address such a large scale optimization problem because they have to sweep all of the training data in each dictionary updating procedure. Inspired by \cite{MBPS09,MBPS10}, an \emph{online} algorithm with less complexity and memory is introduced to address the same learning problem shown in \cite{DCS09,BLLLJC15} but on a large dataset.\footnote{In this paper, large or large-scale dataset means this dataset contains a large amount of training data, i.e., $P$ is very large.} { We use a toy example to briefly explain the benefit of training on a large-scale dataset. Assume that the dimension of the dictionary is $64\times 100$ and the number of non-zeros in the sparse vector $\vtheta$ is $4$. Then the number of subspaces in this dictionary attains $\binom{100}{4}\approx 3.9\times 10^6$. Thus, we see such a dictionary provides a rich number of subspaces which motives us to train the dictionary on a large-scale dataset to explore the dictionary to represent the signal of interests better. One can still imagine that along with the increase of the dimension of the dictionary, the number of subspaces will become much richer and we can expect such a dictionary may yield many interesting properties. Indeed, the benefit of learning a dictionary on a large dataset or a high dimension (without training the sensing matrix) has been experimentally demonstrated in \cite{Jere16,Jere16SPL,RHGZ16}. Moreover, the simulations shown in this paper also indicate the merit of learning the CS system (both the dictionary and the sensing matrix) on a large-scale dataset.}% which also considers both the SRE and projected SRE.

%Now when we double the dimension of the dictionary to $128\times 128$,  the number of possible subspaces of dimension 4 introduced by this dictionary becomes $\binom{128}{4}\approx 1.1\times 10^7$.
%

Note that, in each step, the sensing matrix is either updated with an iterative algorithm in \cite{DCS09} or an alternating-minimization method\footnote{Though each step has a closed-form solution, it requires one SVD in each iteration.} in \cite{BLLLJC15}, both requiring many times of SVDs. To overcome this issue, we suggest an efficient method to optimize the sensing matrix which is robust to the SRE. The proposed method is inspired by the recent results in \cite{LLLBJH15,HBLZ16,TZ16} for robust sensing matrices, but it differs from these works in which there is no need to tune the trade-off parameter and hence it is more suitable for online learning and dynamic data. The experiments on natural images demonstrate that jointly optimizing the Sensing Matrix and Sparsifying Dictionary (SMSD) on a large dataset has much better performance in terms of signal recovery accuracy than with the ones shown in \cite{DCS09,BLLLJC15}. {Notice that in this paper we want to design a CS system for the applications where the SRE exists in which is the case for the natural images.} %So we focus the signal of interest on real images in this paper for convenience.}

The rest of this paper is organized as follows.  In Section \ref{S_2}, a novel model is proposed to design the sensing matrix to reduce the coherence\footnote{The coherence between two vectors $\bm a, \bm b \in\Re^{M}$ is defined as $\frac{\bm a^{\cal T}\bm b}{\|\bm a\|_2 \|\bm b\|_2}$.}  between each two columns in $\bm\Phi\bm\Psi$  and overcome the influence of SRE. Moreover, a closed-form solution is derived to obtain the optimized sensing matrix which is parameter free and then more suitable for the following joint learning SMSD method. A joint optimization algorithm for learning SMSD on a large dataset is suggested in Section \ref{S_3}. For learning the sparsifying dictionary on a large dataset efficiently, an online method is introduced to consider such a large training data.\footnote{Actually, the training data is only involved in \eqref{dictionary_involve_projection}. So the developed online algorithm is only for updating dictionary. For brevity, we call the whole joint algorithm as online SMSD.} Some experiments on natural images are carried out in Section \ref{S_4} to demonstrate the effectiveness of the proposed algorithm and the advantage of training on a large dataset comparing with other methods. Conclusion and future work are given in Section \ref{S_5}.

\section{An Efficient Method for Robust Sensing Matrix Design}\label{S_2}
In this section, we present an efficient method to design a robust\footnote{Following the terminology used in \cite{LLLBJH15,TZ16}, a robust sensing matrix refers to a sensing matrix who yields robust performance for signals whether exist SRE, $\bm e\neq 0$.} sensing matrix. To begin, we note that one of the major purposes in optimizing the sensing matrix is to reduce the coherence between each two columns of the equivalent dictionary $\bm\Phi\bm\Psi$.
{This leads to the work \cite{E07,LZYCB13} which demonstrates that the optimized sensing matrix such that the equivalent dictionary with small mutual coherence
yields much better performance than the one with a random sensing matrix for the exactly sparse signals ({i.e.}, $\bm e = \bm 0$ for the signal model~\eqref{sparse_representation}). See also \cite{LH14,MM17} for directly minimizing the mutual coherence of the equivalent dictionary. However, it was recently realized~\cite{LLLBJH15,HBLZ16,TZ16} that such a sensing matrix is not robust to SRE and thus the corresponding CS system results in poor performance in practice, like sampling the natural images, where the SRE exists even when we represent the images with a well designed dictionary. Alternatively, the average mutual coherence (i.e., the coherence on a least square metric instead of the infinity norm) rather than the exact mutual coherence is suggested in \cite{LLLBJH15,HBLZ16} for designing an optimal robust sensing matrix. Specifically, a robust sensing matrix is attained by solving \cite{LLLBJH15,HBLZ16}:
\e
\min_{\bm\Phi}~\|\bm I_L-\bm\Psi^\mathcal T\bm \Phi^\mathcal T\bm\Phi\bm\Psi\|_F^2 + \lambda\|\bm\Phi\bm E\|_F^2\label{Optimized:Projection}
\ee
where $\|\cdot\|_F$ denotes the Frobenius norm and $\bm I_L$ represents an identity matrix with dimension $L$.\footnote{MATLAB notations are adopted in this letter. In this connection, for a vector, $\bm v(k)$ denotes the $k$-th component of $\bm v$. For a matrix, $\bm Q(i,j)$ means the $(i,j)$-th element of matrix $\bm Q$, while $\bm Q(k,:)$ and $\bm Q(:,k)$ indicate the $k$-th row and column vector of $\bm Q$, respectively. $\bm E(:,i)=\bm e_i,~i=1,\cdots,P$, and $\lambda$ is a trade-off parameter that balance the coherence of the equivalent dictionary and the robustness of the sensing matrix to the SRE.}  According to the recent result shown in \cite{TZ16}, it suggests replacing the penalty $\|\bm\Phi\bm E\|_F^2$ by $\|\bm\Phi\|_F^2$ (which is independent of the training data) since  $\|\bm\Phi\|_F^2$ has the same effectiveness as $\|\bm\Phi\bm E\|_F^2$ when the SRE is modelled as the Gaussian noise and $P\rightarrow \infty$. Thus, the robust sensing matrix is developed via addressing  \cite{TZ16}: % in natural image scenario:
\e
\min_{\bm\Phi}~f(\bm\Phi) = \|\bm I_L-\bm\Psi^\mathcal T\bm \Phi^\mathcal T\bm\Phi\bm\Psi\|_F^2+\lambda\|\bm\Phi\|_F^2\label{robust_projection_lambda}
\ee

Numerical experiments with natural images show that the optimized sensing matrix through solving \eqref{robust_projection_lambda} with a well-chosen $\lambda$ yields state-of-the-art performance in CS-based image compression \cite{TZ16}.} However, we note that it is nontrivial to choose an optimal $\lambda$ for \eqref{robust_projection_lambda} since the two terms $\|\bm I_L-\bm\Psi^\mathcal T\bm \Phi^\mathcal T\bm\Phi\bm\Psi\|_F^2$ and $\|\bm\Phi\|_F^2$ have different physical meanings: the formal represents the average mutual coherence of the equivalent dictionary $\bm \Phi\bm \Psi$, while the later is the energy of the sensing matrix $\bm \Phi$. For off-line applications when the dictionary is fixed, it is suggested to choose a $\lambda$ by searching a given range and looking at the performance of the resulted sensing matrices \cite{LLLBJH15,HBLZ16,TZ16} on the testing dataset. However, this strategy becomes very inefficient for online applications when the dictionary $\bm\Psi$ is evolving which belongs to the case in this paper. To avoid tuning the parameter $\lambda$, we suggest designing the robust sensing matrix with the following two steps: find a set of solutions which minimize $\|\bm I_L-\bm\Psi^\mathcal T\bm \Phi^\mathcal T\bm\Phi\bm\Psi\|_F^2$ (i.e., solve \eqref{robust_projection_lambda} without the term $\|\bm \Phi\|_F^2$), and then locate a $\bm\Phi$ among these solutions that has smallest energy. Thus, we consider the following optimization problem to design the sensing matrix which is slightly different from \eqref{robust_projection_lambda}:
\e
\left.\begin{array}{rl}
	\min\limits_{\bm\Phi\in\mathcal S}&\|\bm\Phi\|_F^2\\
	{\mathcal S =}&\arg\min\limits_{\tilde{\bm\Phi}\in\Re^{M\times N}}~~ g(\tilde{\bm\Phi})=\|\bm I_L-\bm\Psi^\mathcal T\tilde{\bm\Phi}^\mathcal T\tilde{\bm\Phi}\bm\Psi\|_F^2
\end{array}\right.\label{robust_projection_free_lambda}
\ee
Let ${\mathcal U}_{M,\bar N}:= \left\{\bm U_{M,\bar N}:\bm U_{M,\bar N}^{\cal T}\bm U_{M,\bar N} = \bm I_{\bar N}\right\}$ denote the set of $M\times \bar N$ orthonormal matrices for $\bar N\leq M$. When $\bar N = M$, to simplify the notation, we use $\mathcal U_{M}$ to denote the set of $M\times M$ orthonormal matrices. The following result establishes a set of closed-form solutions for \eqref{robust_projection_free_lambda}:% {\bf Lemma}:% and $\epsilon$ is the minimal value of the function $G(\bm\Phi)$ with $\text{Rank}(\bm\Phi)=M$.
%Clearly, \eqref{robust_projection_free_lambda} is a special bilevel optimization problem because the upper and low-level variables are the same. Thanks to the special structure of $G(\bm \Phi)$ in \eqref{robust_projection_free_lambda},
\begin{theorem}\label{lemma_1}
	Let $\bm\Psi=\bm U_{\bm\Psi}\bm \Lambda\bm V_{\bm\Psi}^\mathcal T$ be an SVD of $\bm\Psi$, where $\text{Rank}(\bm\Psi)=\bar N\leq N$, $\bm\Lambda=\text{diag}(\lambda_1,\lambda_2,\cdots,\lambda_{\bar N})>0$ with $\lambda_1\geq\lambda_2\geq\cdots\geq\lambda_{\bar N}$, and $\bm U_{\bm \Psi}$ and $\bm V_{\bm \Psi}$ are $N\times \bar N$ and $L\times \bar N$ orthonormal matrices, respectively. When $\bar N\geq M$, a set of optimal solutions for \eqref{robust_projection_free_lambda} is specified by
	\e
	{\cal W}_1:=\left\{\bm\Phi: \bm\Phi= \bmat \bm U_M & \bm 0\emat \bm\Lambda^{-1}\bm U_{\bm \Psi}^\mathcal T, \bm U_M\in {\cal U}_M\right\}
	\label{optimal_robust_projection_free_lambda}
	\ee
	On the other hand, when $\bar N< M$,  a set of optimal solutions for \eqref{robust_projection_free_lambda} is specified by
	\e
	{\cal W}_2:=\left\{\bm\Phi: \bm\Phi= \bm U_{M,\bar N} \bm\Lambda^{-1}\bm U_{\bm \Psi}^\mathcal T, \bm U_{M,\bar N}\in\mathcal U_{M\times \bar N}\right\}
	\label{optimal_robust_projection_free_lambda 2}
	\ee
\end{theorem}
\begin{proof}
	We first rewrite $g({\bm \Phi})$ as
	\begin{align*}
	g(\bm \Phi) &= \|\bm I_L- \bm V_{\bm\Psi}\bm \Lambda\bm U_{\bm\Psi}^\mathcal T\bm\Phi^\mathcal T\bm\Phi \bm U_{\bm\Psi} \bm \Lambda\bm V_{\bm\Psi}^\mathcal T \|_F^2\\
	&= \|\bm I_{\bar N}- \bm \Lambda\bm U_{\bm\Psi}^\mathcal T\bm\Phi^\mathcal T\bm\Phi \bm U_{\bm\Psi} \bm \Lambda \|_F^2 + L - \bar N
	\end{align*}
	Thus, minimizing $g(\bm \Phi)$ is equivalent to
	\e
	\min_{\bm\Phi} h(\bm \Phi) = \|\bm I_{\bar N}- \bm \Lambda\bm U_{\bm\Psi}^\mathcal T\bm\Phi^\mathcal T\bm\Phi \bm U_{\bm\Psi} \bm \Lambda \|_F^2.
	\label{eq:h}\ee
	We proceed by considering the following two cases.
	
	case I: $ \bar N\geq M$. Noting that $\text{rank}(\bm \Lambda\bm U_{\bm\Psi}^\mathcal T\bm\Phi^\mathcal T\bm\Phi \bm U_{\bm\Psi} \bm \Lambda) = M \leq \bar N$ and utilizing the Eckart-Young-Mirsky theorem~\cite{RC13}, we have that $h(\bm \Phi)\geq \bar N - M$
	and it achieves its minimum when\footnote{One can check that $\|\bm I_{\bar N} - \bm U_{\bar N,M}\bm U_{\bar N,M}^{\cal T}\|_F^2 = \bar N -  \text{Tr}(\bm U_{\bar N,M}\bm U_{\bar N,M}^{\cal T}) = \bar N - M$.}
	\[
	\bm \Lambda\bm U_{\bm\Psi}^\mathcal T\bm\Phi^\mathcal T\bm\Phi \bm U_{\bm\Psi} \bm \Lambda = \bm U_{\bar N,M}\bm U_{\bar N,M}^{\cal T},
	\]
	where $\bm U_{\bar N,M}$ is an arbitrary $\bar N\times M$ orthonormal matrix.
	The set of $\bm \Phi$ that satisfies the above equation is given by
	\e
	\bm \Phi \in \mathcal S =\left\{ \bm U_{\bar N,M}^{\cal T} \bm\Lambda^{-1}\bm U_{\bm \Psi}^\mathcal T: \bm U_{\bar N,M}^{\cal T} \bm U_{\bar N,M} = \bm I_M \right\}.
	\label{eq:def S 1}\ee
	With this, we turn to find $\bm \Phi$ such that it has the smallest $\|\bm \Phi\|_F^2$. To that end, we rewrite
	$$
	\|\bm \Phi\|_F^2 =\text{Tr}\left(\bm\Phi^\mathcal T\bm\Phi\right)=\text{Tr}\left(\bm U_{\bar N,M}\bm U_{\bar N,M}^{\cal T} \bm\Lambda^{-2}  \right)=\sum_{i=1}^{\bar N} \frac{\alpha_i}{\lambda_i^2}\\
	$$
	where $\alpha_i$ is the $i$-th diagonal of $\bm U_{\bar N,M}\bm U_{\bar N,M}^{\cal T}$. It is clear that $0\leq \alpha_i\leq 1$ and $\sum_{i=1}^{\bar N} \alpha_i = M$ since $\bm U_{\bar N,M}$ is an $\bar N\times M$ orthonormal matrix. Therefore, we have
	$
	\sum_{i=1}^{\bar N} \frac{\alpha_i}{\lambda_i^2} \geq \sum_{i=1}^{M}\frac{1}{\lambda_i^2}
	$
	and it achieves the minimum value when $\alpha_1 = \cdots = \alpha_M = 1$ and $\alpha_{M+1} = \cdots = \alpha_{\bar N} = 0$. The last condition implies that $\bm U_{\bar N,M} = \bmat \bm U_M\\ \bm 0\emat$ with $\bm U_M$ is an arbitrary  $M\times M$ orthonormal matrix.
	
	case II: $ \bar N< M$. We first note that $h(\bm \Phi)\geq 0$
	and it achieves its minimum when
	\[
	\bm \Phi \in \mathcal S =\left\{ \bm U_{M,\bar N} \bm\Lambda^{-1}\bm U_{\bm \Psi}^\mathcal T: \bm U_{M,\bar N}^{\cal T} \bm U_{\bar N,M} = \bm I_{\bar N} \right\}.
	\]
	where $\bm U_{M,\bar N}$ is an arbitrary $M\times \bar N$ orthonormal matrix. For such $\bm \Phi$, we have
	$$
	\|\bm \Phi\|_F^2 =\text{Tr}\left(\bm\Phi^\mathcal T\bm\Phi\right)=\text{Tr}\left(\bm\Lambda^{-2} \right)= \sum_{i=1}^{\bar N} \frac{1}{\lambda_i^2}
	$$
	which implies that each $\bm \Phi$ has the same energy. %This completes the proof.
\end{proof}

We remark that \cite[Theorem 2]{LZYCB13} also gives the set of optimal solutions that minimizes $g(\bm \Phi)$ for the case $\text{Rank}(\bm\Psi)=\bar N \geq M$, i.e., $\overline {\cal S}=\left\{\bmat \bm U_{M}&\bm 0 \emat \bm \Lambda ^{-1}\bm U_{\bm \Psi}^\mathcal T
,\bm U_{M}\in {\cal U}_M\right\}.
$
By noting that $\bmat \bm U_{M}^\mathcal T\\\bm 0 \emat \in {\cal U}_{\bar N,M}$, it is clear that ${\cal S} = \overline{\cal S}$ where $\cal S$ is given in \eqref{eq:def S 1}.

When $\text{Rank}(\bm\Psi)=\bar N \geq M$ (which is true for most of applications),
\eqref{optimal_robust_projection_free_lambda} gives a set of optimal solutions for \eqref{robust_projection_free_lambda} and implies that there exists some degrees of freedom to choosing $\bm U_M$. The following result investigates the performance of the sensing matrices with different $\bm U_M$.
\begin{lemma}
	Compressive sensing systems with the same dictionary $\bm \Psi$  and different $\bm \Phi\in \mathcal W_1$ (which is defined in \eqref{optimal_robust_projection_free_lambda}) have the same performance.
	\label{lem:equivalence of Phi}
\end{lemma}
\begin{proof}
	Suppose we have two compressive sensing systems with the same dictionary $\bm \Psi$ and the sensing matrices $\bm \Phi = \bmat \bm U_M & \bm 0\emat \bm\Lambda^{-1}\bm U_{\bm \Psi}^\mathcal T$ and $\overline{\bm \Phi} = \bmat \overline{\bm U}_M & \bm 0\emat \bm\Lambda^{-1}\bm U_{\bm \Psi}^\mathcal T$, respectively, where $\bm U_M, \overline {\bm U}_M \in {\cal U}_M$.
	For any $\bm x\in\Re^N$, the two CS systems obtain the measurements as $\bm y = \bm \Phi \bm x, \ \overline {\bm y} = \overline{\bm \Phi} \bm x$
	and respectively attempt to recover $\bm x$ via
	\[
	\min_{\theta}\|\bm y - \bm \Phi\bm \Psi \bm \theta\|^2, ~~\text{s.t.}~~\|\bm\theta\|_0\leq K
	\]
	and
	\[
	\min_{\theta}\|\overline{\bm y} - \overline{\bm \Phi}\bm \Psi \bm \theta\|^2, ~~\text{s.t.}~~\|\bm\theta\|_0\leq K
	\]
	The proof is completed by noting that the above two equations have the same solution since
	$$
	\|\bm y - \bm \Phi\bm \Psi \bm \theta\|^2= \|\overline{\bm U}_M\bm U_M^{\cal T}\left(\bm y - \bm \Phi\bm \Psi \bm \theta\right)\|^2= \|\overline{\bm y} - \overline{\bm \Phi}\bm \Psi \bm \theta\|^2
	$$
\end{proof}
{\bf Lemma \ref{lem:equivalence of Phi}} implies that we can choose a $\bm \Phi$ in \eqref{optimal_robust_projection_free_lambda} with $\bm U_M$ as an identity matrix and it has the same performance as other $\bm \Phi\in{\cal W}_1$. Moreover, by choosing an identity matrix for $\bm U_M$, we can save computations in learning the dictionary and the solution for \eqref{robust_projection_free_lambda} becomes
\e
\bm\Phi = \phi(\bm\Psi) \triangleq \bm\Lambda_M^{-1} \bm U_{\bm \Psi}\left(:,1:M\right)^\mathcal T\label{solution_robust_projection_free_lambda}
\ee
where $\bm\Lambda_M=\bm\Lambda(1:M,1:M)$. Clearly, $\bm\Lambda_M$ is a diagonal matrix and the calculation of its inverse is cheap. In effect, one time SVD of the dictionary $\bm\Psi$ dominates the main complexity in the sensing matrix updating procedure.\footnote{In effect, we only needs the previous largest $M$ singular values and the corresponding left orthogonal matrices. So the computation can be reduced further by utilizing power method.} Compared with the methods shown in \cite{DCS09,BLLLJC15} which need to perform the eigenvalue decomposition or SVD many times, our proposed method already saves significant computations.

{
We end this section by comparing \eqref{solution_robust_projection_free_lambda} with gradient descent solving \eqref{robust_projection_lambda} in terms of the computational complexity, though we note that our main purpose to use \eqref{solution_robust_projection_free_lambda} is to avoid tuning the parameter $\lambda$ in \eqref{robust_projection_lambda}.
The gradient of $f(\mPhi)$ is given as follows
\[
\nabla_{\bm \Phi} f(\bm \Phi)=2\lambda\bm\Phi-4\bm \Phi\bm\Psi\bm\Psi^\mathcal T+4\bm \Phi\bm\Psi\bm\Psi^\mathcal T\bm \Phi^\mathcal T \bm \Phi\bm\Psi\bm\Psi^\mathcal T.
\]
Suppose $\bm \Psi \bm\Psi^\mathcal T$ is precomputed and then evaluating the gradient $\nabla_{\bm \Phi} f(\bm \Phi)$ requires $O(MN^2)$ computations. Thus, the gradient descent has at least $O(MN^2)$ computational complexity, though we are not ensured\footnote{Though \eqref{robust_projection_lambda} is nonconvex, the recent work on low-rank optimization \cite{Zhu2017} indicates gradient descent can converge to the global solution for a set of low-rank optimizations. The convergence is also experimentally verified for \eqref{robust_projection_lambda} in \cite{TZ16}, though there is no theoretical guarantee about the convergence rate (i.e., how fast it converges to the global solution).
} how fast the gradient descent converges. As indicated by (9), our closed-form solution only needs to compute the first $M$ eigenvectors and corresponding eigenvalues of $\bm \Psi\bm \Psi^T$, which has computational complexity of $O(MN^2)$. We also note that in the dictionary updating procedure (see \eqref{Xi_solution} in Section \ref{S_3}), it is required to compute $(\mI_N +\frac{1}{\gamma}\mPhi^\mathcal T\mPhi)^{-1}$, which can be directly obtained through \eqref{solution_robust_projection_free_lambda} (the SVD form of $\bm \Psi$). If we utilize the gradient descent method to update $\bm \Phi$, then we still need to compute such an inverse when updating the dictionary and this can be saved if we utilize \eqref{solution_robust_projection_free_lambda}.
}

%Actually, the computation complexity can also be reduced since only the largest $M$ singular values and the corresponding left singular vectors are required.%\footnote{In MATLAB, we only need the command `svds' to conduct the svd decomposition which is faster than `svd'.}

%Due to the special structure of this solution, we will find it can also save the computations in the dictionary update procedure which will be shown in the next section.

\section{Online Learning SMSD Simultaneously}\label{S_3}

{\color{black}

We begin this section by considering the problem of jointly optimizing the SMSD on a very large training dataset first. Moreover, the corresponding joint optimization problem is solved via the alternating-minimization based approach. In order to reduce the complexity of learning a dictionary on such a large training data, an online algorithm with the consideration of the influence of the projected SRE, i.e., $\|\mPhi\ve\|_2$, is developed.

%Moreover, we will also show the advantage of using the developed method in Section \ref{S_2} for jointly learning the SMSD.

\subsection{Online Joint SMSD Optimization}
Given a set of $P$ training signals $\bm X(:,k)=\bm x_k$, $k=1,2,\cdots,P$, our purpose here is to jointly design the SMSD. To this end, a proper framework is required. Classical dictionary learning attempts to minimize the following sparse representation error (SRE):
\e
\min\limits_{\bm\Psi\in\mathcal{C},\bm\Theta} \|\bm X-\bm\Psi\bm\Theta\|_F^2,~~\text{s.t.}~~\|\bm\theta_k\|_0\leq K,~\forall k\label{dictionary_learning}
\ee
where $\bm\Theta(:,k) = \bm\theta_k$, $\forall k$ contains the sparse coefficient vectors and $\mathcal C$ is a constraint set to avoid trivial solutions. %force each column of $\bm\Psi$ a unit norm.

Note that in CS, we obtain the linear measurements $\vy$ as in \eqref{eq:y} and then recover the signal from $\vy$ by first recovering the sparse coefficients $\bm\theta$ and then obtain $\vx$ via $\mPsi\vtheta$. Therefore, a smaller $\bm\Phi \ve$ is also preferred.
This implies that besides reducing the SRE $\|\bm X-\bm\Psi\bm\Theta\|_F^2$, giving a sensing matrix $\mPhi$, the dictionary is also expected to reduce the projected SRE $\|\bm \Phi (\bm X - \bm \Psi \bm \Theta)\|_F^2$ \cite{DCS09}-\cite{ZhuCCS}. Now the sensing matrix and the sparsifying dictionary are jointly optimized by \cite{DCS09,BLLLJC15}
\e
\begin{array}{rl}
	\min\limits_{\bm\Psi\in\mathcal C,\bm\Theta,\bm \Phi}&\gamma\|\bm X-\bm\Psi\bm\Theta\|_F^2+\|{\bm\Phi \bm X}-\bm\Phi\bm\Psi\bm\Theta\|_F^2\\
	\text{s.t.}&  \bm \Phi = \phi (\bm \Psi), \|\bm\theta_k\|_0\leq K,~\forall k
\end{array}
\label{SMSD}
\ee
where $\phi(\bm \Psi)$ is given in \eqref{solution_robust_projection_free_lambda} and $\gamma\in\left[0,1\right]$ is a trade-off parameter to balance the SRE and the projected SRE. The value of $\gamma$ can be determined through grid search to receive the highest signal recovery accuracy on the testing dataset.

\noindent{\emph{Remark 3.1:}}
\vspace{-0.25cm}
\begin{itemize}
\item We first note that the projected SRE $\|\bm \Phi (\bm X - \bm \Psi \bm \Theta)\|_F^2$ also involves the sensing matrix $\bm \Phi$ and hence this term should also be considered in designing the sensing matrix. As we explained in Section~\ref{S_2}, the sensing matrix $\phi(\bm \Psi)$ given in \eqref{solution_robust_projection_free_lambda} already incorporates the projected SRE. This suggests the advantages of our proposed method for designing the sensing matrix compared with the ones utilized in \cite{DCS09,BLLLJC15} where the projected SRE is not considered.
\vspace{-0.25cm}
\item Compared with a separate approach that (usually) first learns the dictionary by \eqref{dictionary_learning} and then designs the sensing matrix with the learned dictionary, jointly learning the SMSD via \eqref{SMSD} is expected to yield a better CS system as the projected SRE is also minimized sequentially. We refer \cite{DCS09,BLLLJC15} for more discussions regarding the advantages of this joint approach.
\end{itemize}
\vspace{-0.25cm}

Similar to \cite{DCS09,BLLLJC15},  we utilize the alternating-minimization based method for solving the above joint optimization problem \eqref{SMSD}. The main idea is to alternatively update the sensing matrix (when the dictionary and sparse coefficients are fixed) by \eqref{solution_robust_projection_free_lambda} which is cheap and update the dictionary and the sparse coefficients by minimizing the objective function in \eqref{SMSD} when the sensing matrix is fixed, i.e.,
\e
\begin{array}{rl}
	\min\limits_{\bm\Psi\in\mathcal C,\bm\Theta}&\sigma(\bm\Psi,\bm\Theta)\triangleq\gamma\|\bm X-\bm\Psi\bm\Theta\|_F^2+\|\bm Y-\bm\Phi\bm\Psi\bm\Theta\|_F^2\\
	\text{s.t.}&\|\bm\theta_k\|_0\leq K,~\forall k
\end{array}
\label{dictionary_involve_projection}
\ee
where $\bm Y = \bm \Phi \bm X$. As we suggest utilizing a large training dataset to learn the dictionary, an online algorithm is proposed in next subsection to fit such a large-scale case. We depict  the detailed steps for solving \eqref{SMSD} in {\bf Algorithm \ref{Alg_joint_projection_dictionary}}. Compared with the methods in \cite{DCS09,BLLLJC15}, {\bf Algorithm \ref{Alg_joint_projection_dictionary}} is more suitable for working on a large training dataset as it utilizes \eqref{solution_robust_projection_free_lambda} for optimizing the sensing matrix and an online method (in next subsection) which is independent to the size of
training dataset for learning the dictionary. As we stated before, the disadvantage of the methods in \cite{DCS09,BLLLJC15} for solving \eqref{dictionary_involve_projection} is that they have to sweep all of the training data in each iteration which requires extremely high computations and memory if the training dataset becomes large. Also, the sensing matrix updated in \cite{DCS09,BLLLJC15} is an iterative algorithm that requires computing many SVDs which is not suitable for online case. The simulation results in the next section illustrate the effectiveness of {\bf Algorithm \ref{Alg_joint_projection_dictionary}}.

}
\begin{algorithm}[htb]
	\caption{Online Joint Optimization of SMSD}%Projection Matrix and Sparsifying Dictionary with considering projection noise}
	\label{Alg_joint_projection_dictionary}
	\begin{algorithmic}[1]
		\REQUIRE ~\\
		Initial dictionary $\bm\Psi_0$, number of iterations $Iter_{sendic}$.
		\lastcon ~\\          %OUTPUT
		The sensing matrix $\bm\Phi$ and the sparsifying dictionary $\bm\Psi$.
		
		\FOR {$i=1$ {\bf to} $Iter_{sendic}$}
		\STATE Update the sensing matrix $\bm\Phi_i$ with fixed $\bm \Psi=\bm\Psi_{i-1}$ by $\phi(\bm \Psi)$ (which is specified in \eqref{solution_robust_projection_free_lambda}) and compute the two matrices $\bm\Xi_1$ and $\bm \Xi_2$
		\STATE Solve \eqref{dictionary_involve_projection} through \eqref{Dic:surrogate:easy} by {\bf Algorithm \ref{Alg_dictionary_involve_projection}} to update the dictionary $\bm\Psi$ with fixed $\bm\Phi=\bm\Phi_i$
		\ENDFOR
		\RETURN $\bm \Phi$ and $\bm\Psi$
	\end{algorithmic}
\end{algorithm}

\subsection{Online Dictionary Learning with Projected SRE}
In this subsection, we suggest an online algorithm ({\bf Algorithm \ref{Alg_dictionary_involve_projection}})  to overcome the disadvantages of the methods in \cite{DCS09,BLLLJC15} for solving \eqref{dictionary_involve_projection} when the dataset is large. The online method for solving \eqref{dictionary_involve_projection} contains two main stages: firstly, the sparse coefficient vectors in $\bm\Theta$ are computed with a fixed $\bm\Psi$ and then the sparsifying dictionary $\bm\Psi$ is updated with a fixed $\bm\Theta$.\footnote{Note that only randomly part of the training data is sampled during each iteration in our case which is different from the methods shown in \cite{DCS09,BLLLJC15}.} The detailed steps of the online algorithm are summarized in {\bf Algorithm \ref{Alg_dictionary_involve_projection}}.

\begin{algorithm}[htb]
	\caption{Online Dictionary Learning with Projected SRE}% with considering projection noise}
	\label{Alg_dictionary_involve_projection}
	\begin{algorithmic}[1]
		\REQUIRE ~\\
		Training data $\bm X\in \Re^{N\times P}$, trade-off parameter $\gamma$, initial sensing matrix $\bm\Phi$ and dictionary $\bm\Psi_0$, batch size $\eta\geq1$, the sparsity level $K$, the power parameter $\rho$, number of iterations $Iter_{dic}$.
		\lastcon ~\\          %OUTPUT
		Dictionary $\bm\Psi$.
		%\ENSURE
		\STATE $\bm A_0 \leftarrow \bm 0$, $\bm B_0 \leftarrow \bm 0$, $i\leftarrow 1$
		%\STATE shuffle $\bm X$
		\FOR {$t=1$ {\bf to} $Iter_{dic}$}
		\IF {$i+\eta\leq P$}
		\STATE $\bm X_{t}\leftarrow\bm X(:,i:i+\eta-1)$, $\bm Y_t\leftarrow \bm\Phi\bm X_t$\\
		$i\leftarrow i+\eta$
		\ELSE
		\STATE Shuffle $\bm X$, $i\leftarrow 1$
		\STATE $\bm X_t\leftarrow\bm X(:,i:i+\eta-1)$, $\bm Y_t\leftarrow \bm\Phi\bm X_t$\\
		$i\leftarrow i+\eta$
		\ENDIF
		\STATE Sparse coding
		\e
		\left.\begin{array}{rl}\bm \Theta_{t}= &\arg\min_{\tilde{\bm\Theta}_{t}}\left\|\bmat\sqrt{\gamma} \bm X_{t}\\\bm Y_t\emat-\bmat\sqrt{\gamma}~ \bm\Psi_{t-1}\\ \bm\Phi\bm\Psi_{t-1}\emat\tilde{\bm\Theta}_t\right\|_F^2 \\
			&\text{s.t.}~\|\tilde{\bm\Theta}_t(:,k)\|_0\leq K, \forall k
		\end{array}\right.\label{sparse_coding}
		\ee
		\STATE $\bm A_t\leftarrow (1-\frac{1}{t})^\rho\bm A_{t-1}+\frac{1}{\eta}\bm \Theta_{t}\bm \Theta_{t}^\mathcal T$\label{Alg_dictionary_involve_projection:A_t}
		\STATE $\bm B_t\leftarrow (1-\frac{1}{t})^\rho\bm B_{t-1}+\frac{1}{\eta}\bm X_{t}\bm \Theta_{t}^\mathcal T$\label{Alg_dictionary_involve_projection:B_t}
		%\STATE $\bm C_t \leftarrow \bm\Phi\bm B_t$\label{Alg_dictionary_involve_projection:C_t}
		\STATE Compute $\bm\Psi_t$ using {\bf Algorithm \ref{Alg_dictionary_updating}} with $\bm\Psi_{t-1}$  as the initial value, so that
		\en
		%\left.\begin{array}{rl}
		\bm\Psi_t=\arg\min_{\bm\Psi\in\mathcal C}\hat \sigma_t(\bm\Psi)
		%\sigma(\bm\Psi)\triangleq&\frac{1}{2}\sum_{i=1}^t\left(\gamma\|\bm X_i-\bm \Psi_{i-1}\bm\Theta_i\|_F^2\right.\\
		%&\left.+\|\bm Y_i-\bm\Phi\bm\Psi_{i-1}\bm\Theta_{i}\|_F^2\right)\\
		%\Rightarrow&\frac{1}{2}\text{Tr}\left(\bm \Psi^\mathcal T\bm \Psi\left(\gamma\bm A_t\right)\right)-\text{Tr}\left(\bm \Psi^\mathcal T\left(\gamma\bm B_t\right)\right)\\
		%&+\frac{1}{2}\text{Tr}\left(\bm\Psi^\mathcal T\bm\Phi^\mathcal T\bm\Phi\bm\Psi\bm A_t\right)-\text{Tr}\left(\bm\Psi^\mathcal T\bm\Phi^\mathcal T\bm\Phi\bm B_t\right)
		% \end{array}\right.\label{dictionary_updating_problem}
		\een
		
		\ENDFOR
		\RETURN $\bm\Psi_{Iter_{dic}}$ (learned dictionary)
	\end{algorithmic}
\end{algorithm}
%\footnote{Although OMP is not the most efficient algorithm to stand the sparse coding mission, it is simple and the central point in this letter is to show the merit of utilizing the online version to address \eqref{dictionary_involve_projection} on a large dataset. So OMP is chosen to conduct the sparse coding mission throughout this letter. Some other efficient algorithms for sparse coding can be found in \cite{ZE10}.}

For simplicity, similar to \cite{DCS09,BLLLJC15}, {\bf Algorithm \ref{Alg_dictionary_involve_projection}}  utilizes Orthogonal Matching Pursuit (OMP) for addressing the sparse coding problem \eqref{sparse_coding} to update $\bm \Theta$. In the dictionary updating procedure, we are going to solve the following surrogate function in $t$-th iteration instead of considering \eqref{dictionary_involve_projection} directly:

\e
\min_{\bm\Psi} \sigma_t(\bm\Psi)\triangleq \frac{1}{2}\sum_{i=1}^t\left(\gamma\|\bm X_i-\bm \Psi\bm\Theta_i\|_F^2+\|\bm Y_i-\bm\Phi\bm\Psi\bm\Theta_{i}\|_F^2\right)\label{Dic:surrogate}
\ee
Clearly, \eqref{Dic:surrogate} is equivalent  to the following problem:
\e
\min_{\bm\Psi} \hat \sigma_t(\bm\Psi)\triangleq \frac{1}{2}\text{Tr}\left(\bm \Psi^\mathcal T\bm\Omega\bm \Psi\bm A_t\right)-\text{Tr}\left(\bm \Psi^\mathcal T\bm\Omega\bm B_t\right)\label{Dic:surrogate:easy}
\ee
where $\bm A_t= \sum_{i=1}^t\bm \Theta_{i}\bm \Theta_{i}^\mathcal T$, $\bm B_t=\sum_{i=1}^t\bm X_{i}\bm \Theta_{i}^\mathcal T$, $\bm\Omega=\bm I_N+\frac{1}{\gamma}\bm\Phi^\mathcal T\bm\Phi$ and $\text{Tr}(\cdot)$ denotes the trace operator.
Here, we intend to utilize the block-coordinate descent algorithm to update the dictionary column by column.\footnote{\color{black}Here we choose block-coordinate descent because 1) it is parameter-free and does not require tuning any learning rate which is required by stochsatic gradient descent; and 2) there is no need to calculate the inversion of some matrices and only some simple algebra operations are involved.} Specially, the gradient of \eqref{Dic:surrogate:easy} with respect to $j$-th column of $\bm\Psi$ is
\e
\frac{\partial\hat \sigma_t(\bm\Psi)}{\partial\bm\psi_j}=\bm\Psi\bm a_j-\bm b_j+\frac{1}{\gamma}\bm\Phi^\mathcal T\bm\Phi\bm\Psi\bm a_j-\frac{1}{\gamma}\bm\Phi^\mathcal T\bm\Phi\bm b_j\label{gradient_dictionary_column}
\ee
where $\bm\psi_j$, $\bm a_j$ and $\bm b_j$ are the $j$-th column of the matrices $\bm\Psi$, $\bm A_t$ and $\bm B_t$, respectively. Forcing \eqref{gradient_dictionary_column} to be zero, the $j$-th column of $\bm\Psi$ should be updated as in \eqref{dictionary_updating_block_updating}  while keeping the others fixed.
%\en
%\left.\begin{array}{rcl}
%\bm\psi_j&=&\bm\Xi_1\left[\frac{1}{\bm A_l(j,j)}\left(\bm b_j-\bm\Psi\bm a_j\right)+\bm \psi_j\right]+\frac{1}{\bm A_l(j,j)\gamma}\bm\Xi_2\bm c_j\\
%&&+\bm\Xi_3\left[\frac{1}{\gamma}\bm\psi_j-\frac{1}{\bm A_l(j,j)\gamma}\bm\Psi\bm a_j\right]
%\end{array}\right.%\label{column_dictionary_solution}
%\een
The matrices $\bm \Xi_1$ and $\bm\Xi_2$ are equivalent to $\bm\Omega^{-1}$, and $\bm\Omega^{-1}\bm\Phi^\mathcal T\bm\Phi$, respectively. Due to the special structure of $\bm\Phi$ shown in \eqref{solution_robust_projection_free_lambda}, the matrices $\bm \Xi_1$ and $\bm \Xi_2$ can be evaluated simply by:
\e
\left.\begin{array}{rl}
	\bm\Xi_1=&\bm U_{\bm\Psi}\bmat\left(\gamma^{-1}\bm \Lambda_M^{-2}+\bm I_M\right)^{-1}&\bm 0\\\bm0&\bm I_{N-M}\emat\bm U_{\bm\Psi}^\mathcal T\\
	
	\bm\Xi_2=&\bm U_{\bm\Psi}\bmat\left(\gamma^{-1}\bm I_M+\bm\Lambda_M^2\right)^{-1}&\bm 0\\ \bm 0&\bm 0\emat \bm U_{\bm\Psi}^\mathcal T
\end{array}\right.\label{Xi_solution}
\ee
Although we still need to compute the inverse of matrices in \eqref{Xi_solution}, the computational burden becomes cheap here because the related matrices are diagonal matrices.  The detailed steps of updating the dictionary are proposed in {\bf Algorithm \ref{Alg_dictionary_updating}}. Each column of the dictionary is then normalized to have a unit $\ell_2$ norm  to avoid the trivial solution. Following, we suggest several remarks which is useful in practice to improve the performance of {\bf Algorithm \ref{Alg_dictionary_involve_projection}}.

\noindent{\emph{Remark 3.2:}}
\vspace{-0.25cm}
\begin{itemize}
	\item When the training dataset has finite size (though it maybe very large), we suggest simulating the random sampling of the data by cycling over a randomly permuted dataset, i.e., Steps $3$ to $8$ shown in {\bf Algorithm \ref{Alg_dictionary_involve_projection}}.
	\vspace{-0.25cm}
	\item As introduced before, we sample one example from the training dataset instead of swapping all of the data during each iteration. A typical strategy which can be used to accelerate the algorithm is to sample a relatively large examples instead of only one example ($\eta>1$). This belongs to a classical heuristic strategy in stochastic gradient descent method \cite{LB10} called mini-batch which is also useful in our case. Another useful strategy to accelerate the algorithm is to add the history into $\bm A_t$ and $\bm B_t$ as we already shown the formulation of $\bm A_t$ and $\bm B_t$ through the accumulation of $\bm \Theta_t$ and $\bm X_t$. Meantime, we can imagine that the dictionary will approach to a stationary point after necessary iterations.\footnote{Following, we will see that such an observation is compatible with our convergence analysis.} So the latest $\bm\Theta_t$ is more important than the old one. According to such an observation, a forgetting factor is added in $\bm A_t$ and $\bm B_t$ to deemphasize the older information in $\mA_t$ and $\mB_t$ because we want the latest one to dominate the information in $\bm A_t$ and $\bm B_t$. To reach such a purpose, we set the forgetting factor to be $(1-\frac{1}{t})^\rho$ in updating $\bm A_t$ and $\bm B_t$. The detailed formulation can be found in Steps \ref{Alg_dictionary_involve_projection:A_t} and \ref{Alg_dictionary_involve_projection:B_t} in {\bf Algorithm \ref{Alg_dictionary_involve_projection}}. Typically, $\rho$ is set to be larger than $1$.
	\vspace{-0.25cm}
	\item In practical situation, the dictionary learning technique will lead to a dictionary whose atoms are never (or very seldom) used in sparse coding procedure, which happens typically with a not well designed initialization. If we encounter such a phenomenon, one training example is randomly sampled to replace such an atom in this paper.  	
\end{itemize}
\vspace{-0.25cm}
%{\bf Some Remarks...}

%Remind that sparse coefficients and training data in the previous iterations, \ie, $1,~2,~t-1$ are also utilized in the $t$-th iteration which can help to accelerate the convergence of the online algorithm as suggested in \cite{MBPS09}. However, it is clear that the latest iterations have more accurate information, so we decrease the weighting of previous iterations in our algorithm as the steps \ref{Alg_dictionary_involve_projection:A_t}, \ref{Alg_dictionary_involve_projection:B_t}, \ref{Alg_dictionary_involve_projection:C_t} in {\bf Algorithm \ref{Alg_dictionary_involve_projection}}.

\begin{algorithm}[!htb]
	\caption{Dictionary Update}% with considering projection noise}
	\label{Alg_dictionary_updating}
	\begin{algorithmic}[1]
		\REQUIRE ~\\
		{$\bm A_{t-1}=\left[\bm a_1,\cdots,\bm a_L\right], ~\bm B_{t-1}=\left[\bm b_1,\cdots,\bm b_L\right]$, $\bm\Xi_1, ~\bm\Xi_2$,\\
			$\bm \Psi_{t-1}=\left[\bm \psi_1,\cdots,\bm\psi_L\right]$.}
		\lastcon ~\\          %OUTPUT
		Dictionary $\bm\Psi_l$.
		\REPEAT
		\FOR {$j=1$ {\bf to} $L$}
		\STATE Update the $j$-th column to optimize \eqref{Dic:surrogate:easy}:
		\e
		\left.\begin{array}{rcl}
			\bm u_j&\leftarrow&\bm\Xi_1\left[\frac{\bm b_j-\bm\Psi_{t-1}\bm a_j}{\bm A_{t-1}(j,j)}+\bm\psi_j\right]+\\[5pt]
			&&\bm \Xi_2\left[\frac{\bm bj}{\bm A_{t-1}(j,j)\gamma}+\frac{1}{\gamma}\bm \psi_j-\frac{\bm\Psi_{t-1}\bm a_j}{\bm A_{t-1}(j,j)}\right]\\[5pt]
			\bm\Psi_{t-1}(:,j)&\leftarrow&\frac{\bm u_j}{\|\bm u_j\|_2}
		\end{array}\right.\label{dictionary_updating_block_updating}
		\ee
		\ENDFOR
		\UNTIL
		\RETURN $\bm \Psi_{t-1}$ (updated dictionary)
	\end{algorithmic}
\end{algorithm}

\subsection{Convergence Analysis}
Although the logic in our algorithm ({\bf Algorithm \ref{Alg_joint_projection_dictionary}}) is relatively simple, it is nontrivial to prove the convergence of {\bf Algorithm $3$} because of its stochastic nature, the non-convexity and two different objective functions (\eqref{robust_projection_free_lambda} and \eqref{dictionary_involve_projection}). In what follows, we provide the convergence analysis for each step in {\bf Algorithm $3$}. To that end, notice that {\bf Algorithm $3$} contains two parts: optimizing the sensing matrix and learning the sparsifying dictionary to decrease the coherence of $\bm\Phi\bm\Psi$ and to minimize the sparse representation error, respectively. Separately, we claim both of these two steps are convergent.\footnote{The simulation result shown in the next section indicates{\bf Algorithm \ref{Alg_dictionary_updating}} is convergent. However, we left the whole proof for future work.} For the sensing matrix updating procedure, we attain the minimum with one step because of the closed-form solution. Following, we need to investigate whether the updating procedure in dictionary is also convergent. In fact, such a convergence is hold by the following assumptions and propositions which are originally from \cite{MBPS09,MBPS10}.\\[2pt]
\noindent{\bf Assumptions:}
\vspace{-6pt}
\begin{itemize}
	\item[(1).] {\emph {The data admits a distribution with compact support $K$.}}
	\vspace{-0.25cm}
	\item[(2).] {\emph{The quadratic surrogate functions $ \hat \sigma_t$ (defined in \eqref{Dic:surrogate:easy}) are strictly convex with lower-bounded Hessians}}. Assume that the matrix $\bm A_t$ is positive definite. In fact, this hypothesis is in practice verified experimentally after a few iterations of the algorithm when the initial dictionary is reasonable. Specially, all of atoms will be chosen at least once in the sparse coding procedure during the whole iterations. The Hessian matrix of $\hat \sigma_t$ is $\bm \Omega \otimes 2\bm A_t$ where $\otimes$ represents the kronecker product. Clearly, the eigenvalues of  $\bm \Omega \otimes 2\bm A_t$ is the product of $\bm \Omega$ and $\bm A_t$'s eigenvalues. This indicates that the Hessian matrix of $\hat \sigma_t$ is positive definite because $\bm \Omega$ is a positive definite matrix which results in the fact that $ \hat \sigma_t$ is a strictly convex function.
\vspace{-0.25cm}
	\item[(3).]  {\emph{A particular sufficient condition for the uniqueness of the sparse coding solution is satisfied.}} Considering our sparse coding mission \eqref{sparse_coding}, we see it exactly shares the same structure as in \cite{MBPS09,MBPS10}. So this assumption is also satisfied in our case.
\end{itemize}

\begin{prop}
	\label{prop:decrease_func}
	\cite[Propostion 2]{MBPS10}
	Assume the assumptions (1) to (3) are hold, then we have
	\vspace{-0.25cm}
	\begin{itemize}
		\item[1.] ${\hat \sigma_t(\bm\Psi_t)}$ convergences almost surely;
		\vspace{-0.25cm}
		\item[2.] $\sigma(\bm\Psi_t)-{\hat \sigma_t(\bm\Psi_t)}$ converges almost surely to $0$;
		\vspace{-0.25cm}
		\item[3.] $\sigma(\bm\Psi_t)$ converges almost surely.
	\end{itemize}
\end{prop}

%{\color{blue}Is $\sigma(\bm \Psi_t) = \sigma(\bm \Psi_t,\bm \Theta_t)$ the one defined in \eqref{dictionary_involve_projection}? If so, seems $\hat \sigma_t(\bm \Psi_t)$ differs with $\sigma(\bm \Psi_t)$ at least a constant, since when we derive \eqref{Dic:surrogate:easy} from \eqref{Dic:surrogate}, we omit some constant term. Or we can probably define $\sigma(\bm \Psi_t)$ similar to  $\hat \sigma_t(\bm \Psi_t)$, but may use a different notation.}
\begin{prop}
	\label{prop:dic_dif}
	\cite[Propostion 3]{MBPS10}
	Under assumptions (1) to (3), the distance between $\bm \Psi_t$ and the set of stationary points of the dictionary learning problem converges almost surely to 0 when $t$ tends to infinity.
\end{prop}
%However, our experiment shows the proposed {\bf Algorithm $3$} decreases the objective value asymptotically. That means the algorithm is stable.
Obviously, these assumptions are also hold in our case. Conclude that the dictionary updating procedure in our case is also convergent. This verifies what we argue at the second term in {\emph{Remarks $1$}} that the dictionary will approach to a stationary point after enough iterations. {Though we have not rigorously proved the convergence of {\bf Algorithm \ref{Alg_joint_projection_dictionary}}, the convergence of the two parts in {\bf Algorithm \ref{Alg_joint_projection_dictionary}} indicates that the proposed algorithm at least is stable because both of these two steps (updating the sensing matrix and the dictionary) are convergent and decrease the value of the corresponding objective functions. The experiment in the following section also demonstrates such a statement. We note that such convergence is not discussed in \cite{DCS09,BLLLJC15}, where the sensing matrix is updated with an iterative algorithm rather than as here with a closed-form solution. The only convergence analysis we are aware of jointly designing sensing matrix and dictionary is in \cite{Li2018joint}, where both the sensing matrix and the dictionary are optimized in the same framework, but the training algorithm is not customized for large-scale applications. As for the convergence of {\bf Algorithm \ref{Alg_joint_projection_dictionary}}, we defer this to the future work.}
\section{Simulation Results}\label{S_4}
Some experiments on natural images are posed in this section to illustrate the performance of the proposed {\bf Algorithm \ref{Alg_joint_projection_dictionary}}, denoted as $CS_{Alg3}$. We also compare our method with the ones given in \cite{DCS09,BLLLJC15} which also share the same framework as ours but are based on the batch method (sweep the whole training data in each iteration). Although \cite{BLLLJC15} developed the closed-form solutions for each updating procedures, it is still inefficient for the case when the training dataset is large. The methods given in \cite{DCS09,BLLLJC15} are denoted as $CS_{S-DCS}$ and $CS_{BL}$, respectively. Both training and testing data are extracted from the LabelMe database \cite{RTF}. All of the experiments are carried out on a laptop with Intel(R) i7-6500 CPU @ 2.5GHz and RAM 8G.

The signal reconstruction accuracy is evaluated in terms of Peak Signal to Noise Ratio (PSNR) given in \cite{E10}
\en
\varrho_{psnr}\triangleq 10\times \log10\left[\frac{\left(2^r-1\right)^2}{\varrho_{mse}}\right]dB
\een
with $r=8$ bits per pixel and $\varrho_{mse}$ defined as
\en
\varrho_{mse}\triangleq \frac{1}{N\times P}\sum\limits_{k=1}^{P}\|\tilde{\bm x}_k-\bm x_k\|_2^2
\een
where $\bm x_k$ is the original signal, $\tilde{\bm x}_k=\bm\Psi\tilde{\bm\theta}_k$  stands for the recovered signal and $P$ is the number of patches in an image or testing data. The training and testing data are obtained through the following method.

\vspace{3pt}
\noindent \emph{Training data} A set of $8\times 8$ non-overlapping patches is obtained by randomly extracting $400$ patches from each of the images in the whole LabelMe training dataset, with each patch of $8\times 8$ arranged as a vector of $64\times 1$. A set of $400\times2920=1.168\times10^6$ training samples is received for training.

\vspace{3pt}
\noindent\emph{Testing data} The testing data is extracted from the LabelMe testing dataset. Here, we randomly extract $15$ patches from $400$ images and each sample is an $8\times 8$ non-overlapping patch. Finally, we obtain $6000$ testing samples.
\vspace{3pt}

$8\times10^4$ and $6\times10^3$ patches are randomly chosen from the $1.168\times10^6$ \emph{Training data} for $CS_{S-DCS}$ and $CS_{BL}$, respectively, because these two methods cannot stand too large training patches. In order to show the advantage of designing the SMSD on a large training dataset, the same $6\times 10^3$ patches which is prepare for $CS_{BL}$ are also utilized by $CS_{S-DCS}$. For convenience, this case is called $CS_{S-DCS}-small$. The parameters in these two methods are chosen as recommended in their papers. To keep the same dimensions in $\bm\Phi$, $\bm\Psi$ and sparsity level as given in \cite{BLLLJC15}, $M$, $L$ and $K$ are set to $20$, $256$ and $4$ in $CS_{Alg3}$, respectively. The parameters $\gamma$, $\eta$, $Iter_{dic}$ and $Iter_{sendic}$ are set to $\frac{1}{32}$, $128$, $1000$ and $10$ in the proposed {\bf Algorithm \ref{Alg_joint_projection_dictionary}}. The initial sensing matrix and dictionary for \cite{DCS09,BLLLJC15} are a random Gaussian matrix and the DCT dictionary, respectively. The initial sparsifying dictionary in the proposed algorithm is randomly chosen from the training data and the corresponding sensing matrix is obtained through the method shown in Section \ref{S_2}.\footnote{According to our experiments, using the initial value in such a case in our method results in a slightly better performance compared with the initial setting suggested in \cite{DCS09,BLLLJC15}.} The signal recovery accuracy of the aforementioned methods on \emph{testing data} is shown in Fig. \ref{f_1}. The corresponding CPU time of the four cases in seconds are given in Table \ref{t_1}.

\begin{figure}[!htb]
	
	\centering
	\includegraphics[width=0.48\textwidth]{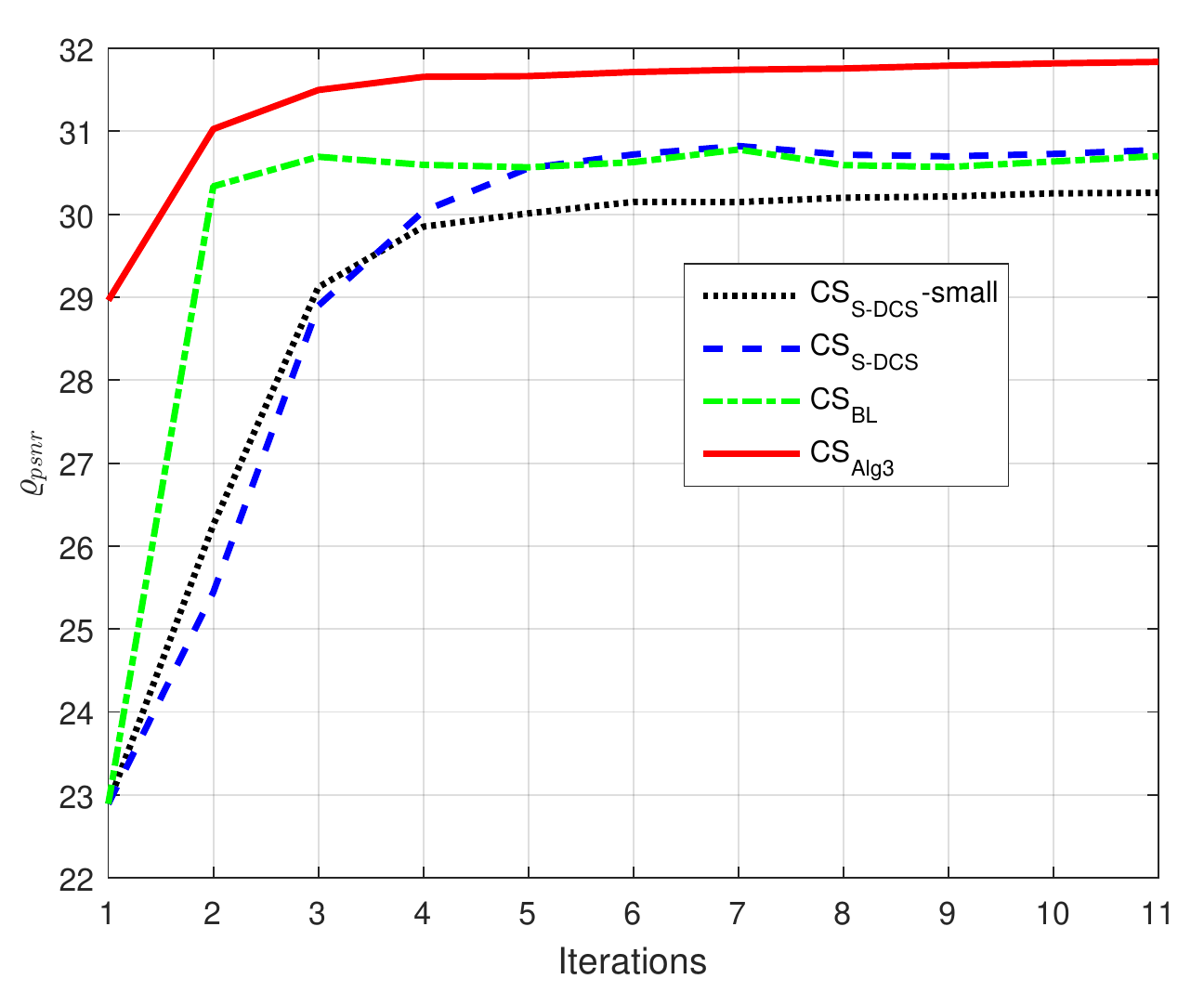}\\
	\caption{The $\sigma_{psnr}$ of the four different cases versus the iteration number on testing data.}\label{f_1} %{\color{red} test on both of the different initial dictionary.}
\end{figure}

\begin{table}[!htb]
	\centering   \caption{The CPU Time Of The Four Different Cases. (Seconds)}\label{t_1}
	\begin{tabular}{|c|c|c|c|}\hline
		$CS_{S-DCS}-small$&$CS_{S-DCS}$&$CS_{BL}$&$CS_{Alg3}$\\\hline
		$2.79\times10^1$&$1.32\times10^3$&$4.33\times10^4$&$1.54\times10^2$\\\hline
	\end{tabular}
\end{table}

Benefiting from the large training dataset, $CS_{S-DCS}$ yields a better performance in terms of $\sigma_{psnr}$ than $CS_{S-DCS}-small$. This indicates that enlarging the training dataset leads to a better CS system. Compared with $CS_{S-DCS}-small$, $CS_{BL}$ has a higher $\varrho_{psnr}$ which meets the observation shown in \cite{BLLLJC15}. However, $CS_{BL}$ needs many SVDs in the algorithm which makes it inefficient and hard to extend to the situation when the training dataset is large. This concern can be observed from Table \ref{t_1} that $CS_{BL}$ needs much more CPU time even for only $6000$ training patches. Although $CS_{BL}$ has a better performance than $CS_{S-DCS}-small$, this advantage will disappear if we enlarge the size of the training dataset in $CS_{S-DCS}$. It can be seen from Fig. \ref{f_1} that $CS_{S-DCS}$ has a similar performance with $CS_{BL}$, but it requires a shorter training time, see Table \ref{t_1}. $CS_{Alg3}$ has a best performance in terms of $\varrho_{psnr}$ compared with other methods. Meantime, $CS_{Alg3}$ has a relatively shorter CPU time but has the largest training dataset. It indicates that {\bf Algorithm \ref{Alg_joint_projection_dictionary}} is suitable for training the SMSD on a large training dataset. Moreover, we observe that training on a large dataset can obtain a better SMSD and the proposed {\bf Algorithm \ref{Alg_joint_projection_dictionary}} belongs to a good choice which takes the efficiency and effectiveness into account simultaneously.

Additionally, we also investigate the performance of the four different CS systems mentioned in this paper on ten natural images. The Structural Similarity Index (SSIM) \cite{ZBSS04} is also involved in comparing the recovered natural images by the different methods. As the results shown in Table \ref{recovered:PSNR:SSIM}, we see the proposed {\bf Algorithm $3$} yields the highest PSNR and SSIM. Compared with $CS_{S-DCS}-small$, $CS_{S-DCS}$ has a higher PSNR and SSIM on all of the ten testing natural images. This meets the argument in this paper that enlarging the size of the training dataset is significant in practice. This can also be illustrated by the methods between $CS_{BL}$ and $CS_{S-DCS}$. Note that $CS_{BL}$ works better than $CS_{S-DCS}-small$ when they have the same small size of training dataset. However, the performance of $CS_{S-DCS}$ will exceed $CS_{BL}$ when the size of the training data is enlarged. All of these imply that training the sensing matrix and the corresponding sparsifying dictionary on a large dataset is preferred. Moreover, the proposed {\bf Algorithm $3$} is a good choice to stand such a mission. To examine the visual effect clearly, the recovered performance of two natural images, i.e., `Lena' and `Mandril' in Fig. \ref{Lena:Mandril:original}, are shown in Fig.s \ref{Lena:reconst} and \ref{Mandril:reconst}.

\begin{table*}[!htb]
	\centering   \caption{Performance Evaluated With Four Cases Shown In This Paper. (Left: PSNR, Right: SSIM. The highest is marked with bold.)}\label{recovered:PSNR:SSIM}
	\begin{tabular}{l||c|c||c|c||c|c||c|c}\hline\hline
		&\multicolumn{2}{|c||}{$CS_{S-DCS}-small$}&\multicolumn{2}{|c||}{$CS_{S-DCS}$}&\multicolumn{2}{|c||}{$CS_{BL}$}&\multicolumn{2}{|c}{$CS_{Alg3}$}\\\hline%\cline{2-11}
		Lena&$33.0566$&$0.9089$&$33.7859$&$0.9184$&$33.3059$&$0.9111$&$\bm{34.6557}$&$\bm{0.9281}$\\\hline
		Elaine&$32.3990$&$0.8073$&$32.6903$&$0.8145$&$32.4076$&$0.8043$&$\bm{33.1756}$&$\bm{0.8244}$\\\hline
		Man&$31.1978$&$0.8738$&$31.8509$&$0.8866$&$31.4686$&$0.8782$&$\bm{32.5941}$&$\bm{0.8999}$\\\hline
		Mandrill&$23.4291$&$0.7598$&$23.8411$&$0.7823$&$23.8221$&$0.7746$&$\bm{24.3753}$&$\bm{0.8007}$\\\hline
		Peppers&$28.9462$&$0.8877$&$29.6975$&$0.9005$&$29.4145$&$0.8925$&$\bm{30.6859}$&$\bm{0.9169}$\\\hline
		Boat&$29.7350$&$0.8561$&$30.3488$&$0.8679$&$30.1027$&$0.8580$&$\bm{31.2858}$&$\bm{0.8837}$\\\hline
		House&$31.5166$&$0.8842$&$32.0602$&$0.8985$&$32.0707$&$0.8956$&\bm{$33.0146}$&$\bm{0.9158}$\\\hline
		Cameraman&$26.2240$&$0.8581$&$26.8272$&$0.8716$&$26.4545$&$0.8673$&$\bm{27.4254}$&$\bm{0.8877}$\\\hline
		Barbara&$25.6148$&$0.8239$&$25.9153$&$0.8316$&$25.5165$&$0.8168$&$\bm{26.0835}$&$\bm{0.8393}$\\\hline
		Tank&$30.7233$&$0.8252$&$30.8210$&$0.8369$&$31.1403$&$0.8361$&$\bm{31.7818}$&$\bm{0.8576}$\\\hline
		Averaged&$29.2842$&$0.8485$&$29.7838$&$0.8609$&$29.5703$&$0.8534$&$\bm{30.5078}$&$\bm{0.8754}$
	\end{tabular}
\end{table*}

\begin{figure}[!htb]
	\centering
	\subfigure[`Lena']{\includegraphics[width=0.23\textwidth]{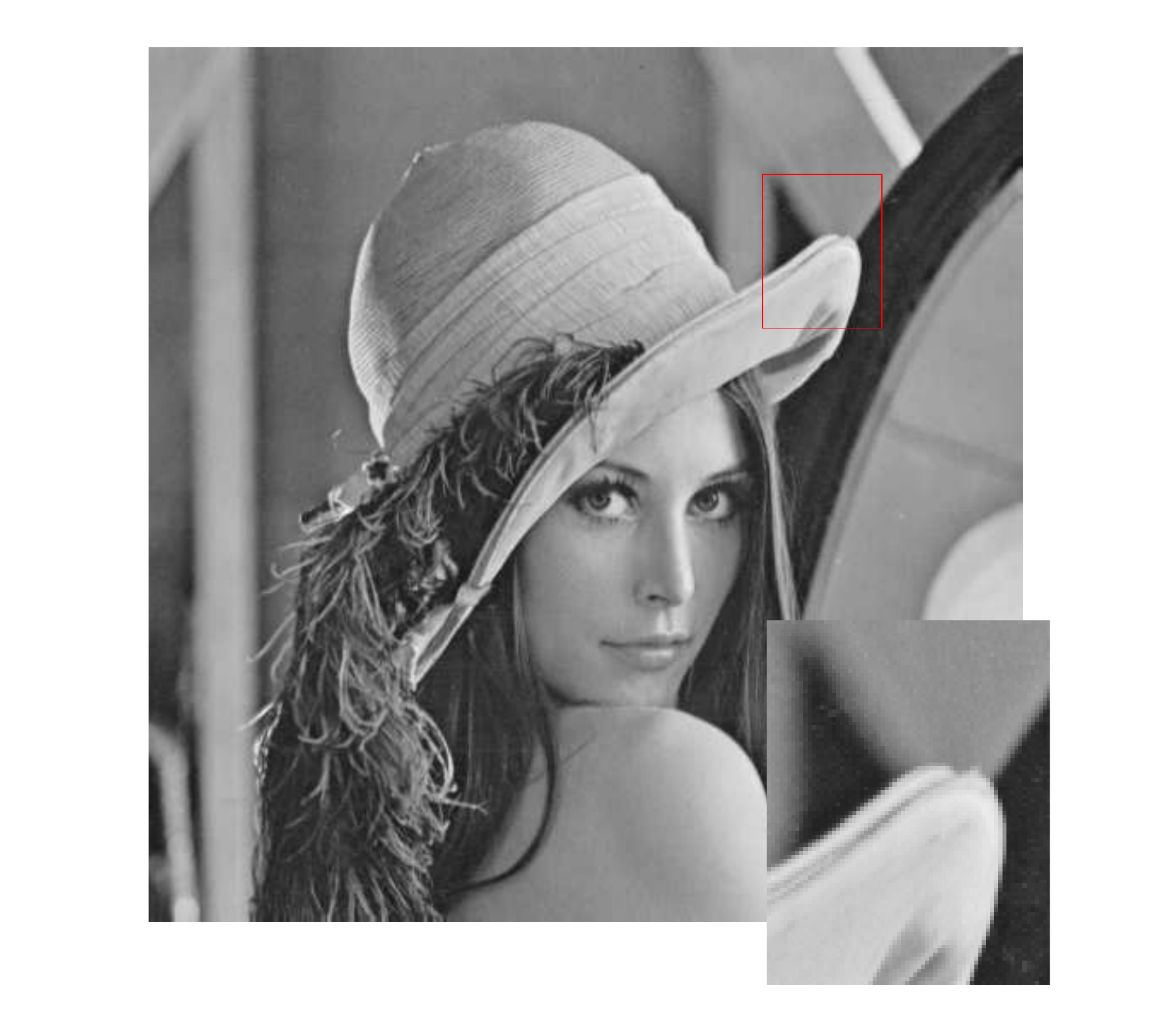}
	}
	\subfigure[`Mandrill']{\includegraphics[width=0.23\textwidth]{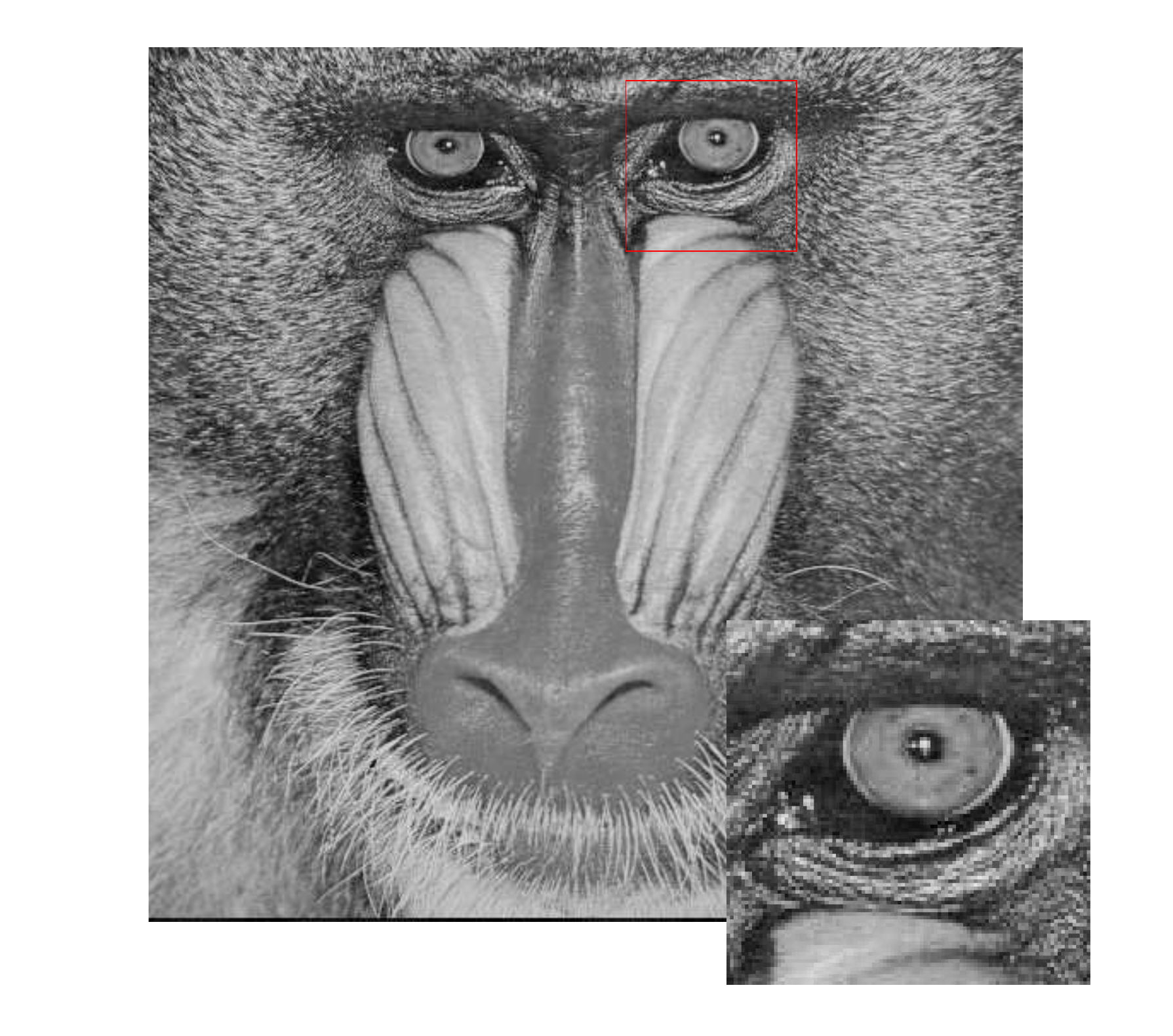}
	}
	\caption{The original testing images.}\label{Lena:Mandril:original}
\end{figure}
\begin{figure}[!htb]
	\centering
	\subfigure[$CS_{S-DCS}-small$]{\includegraphics[width=0.23\textwidth]{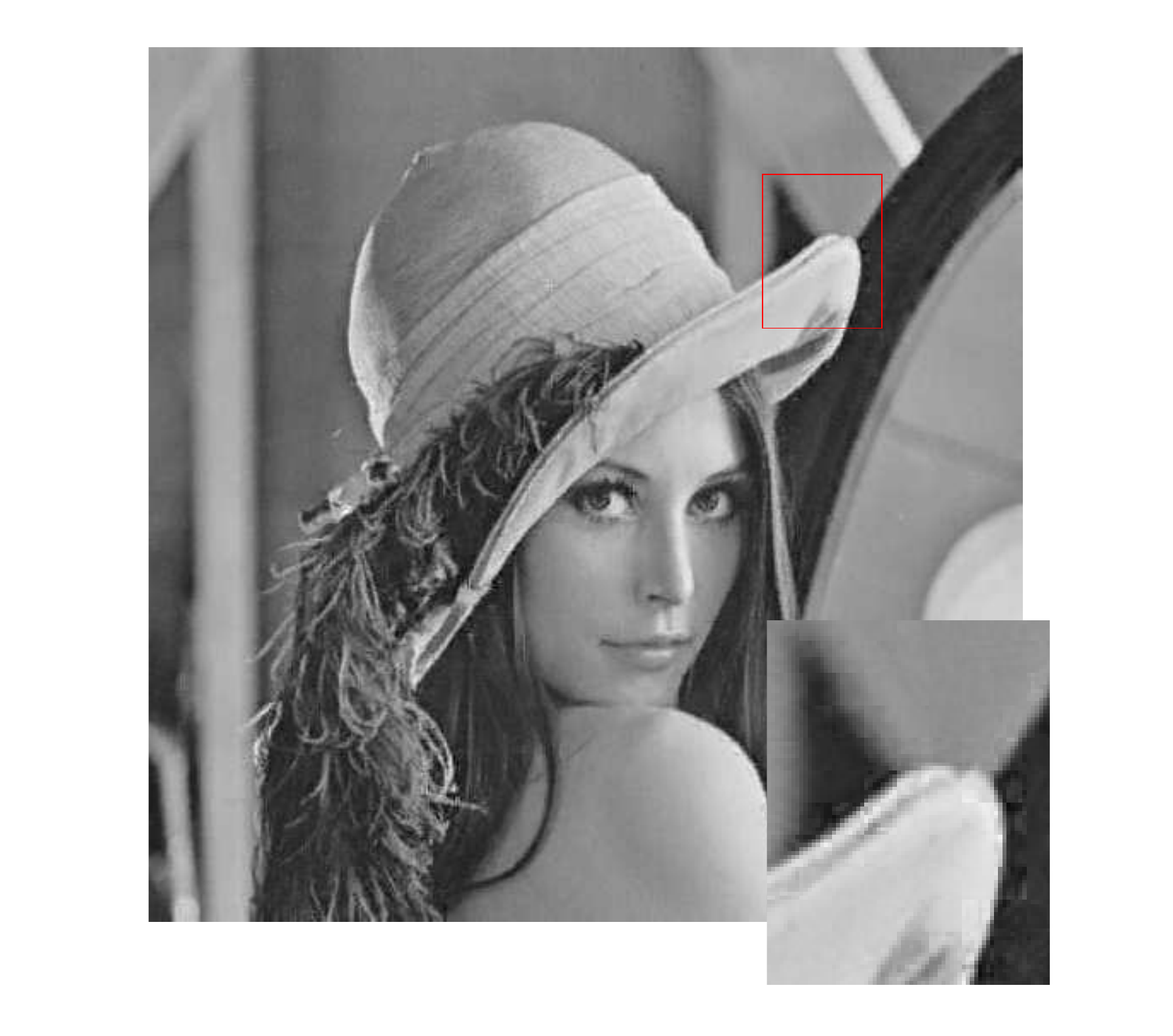}
 	}
	\subfigure[$CS_{S-DCS}$]{\includegraphics[width=0.23\textwidth]{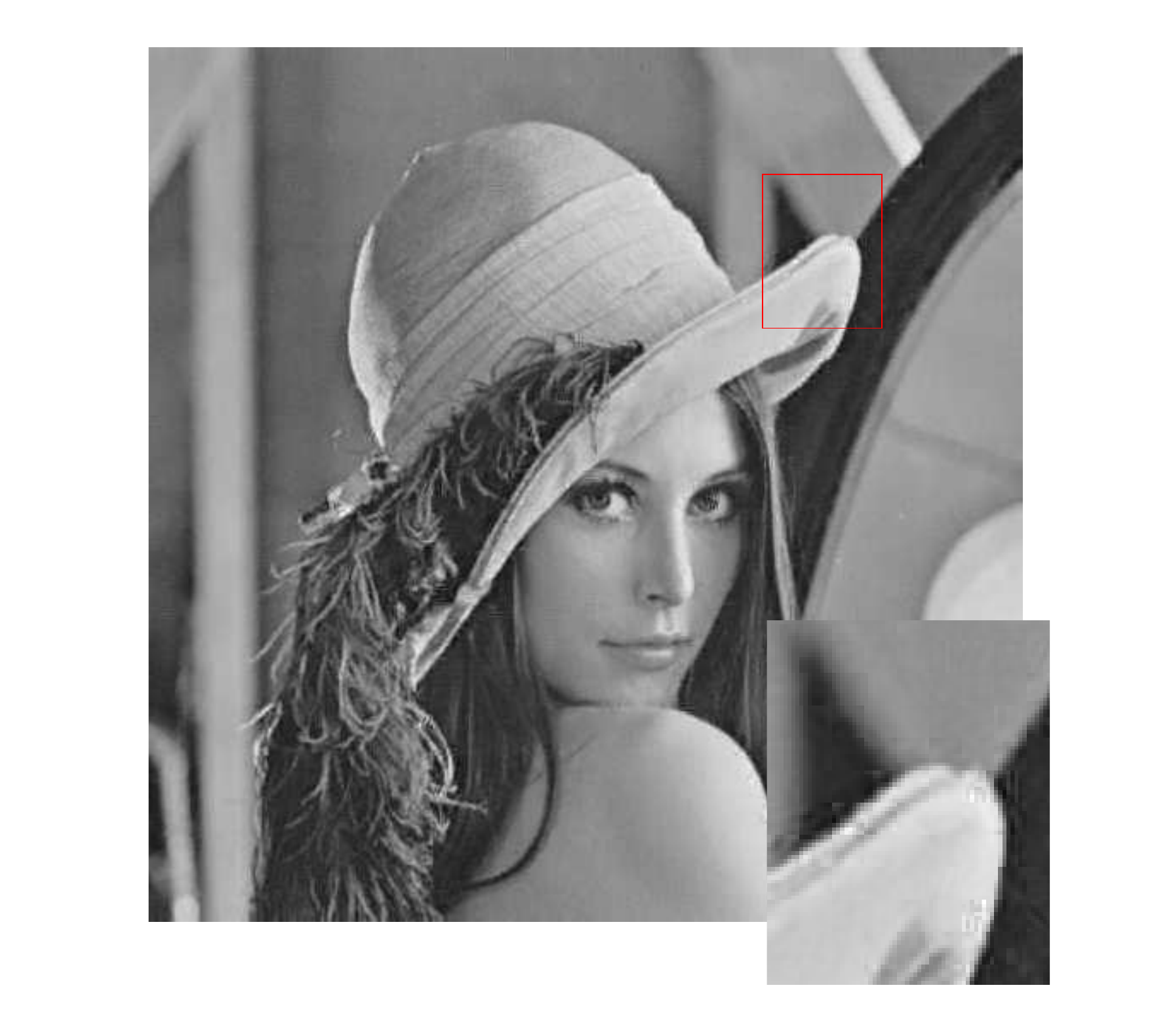}
	}
	\subfigure[$CS_{BL}$]{\includegraphics[width=0.23\textwidth]{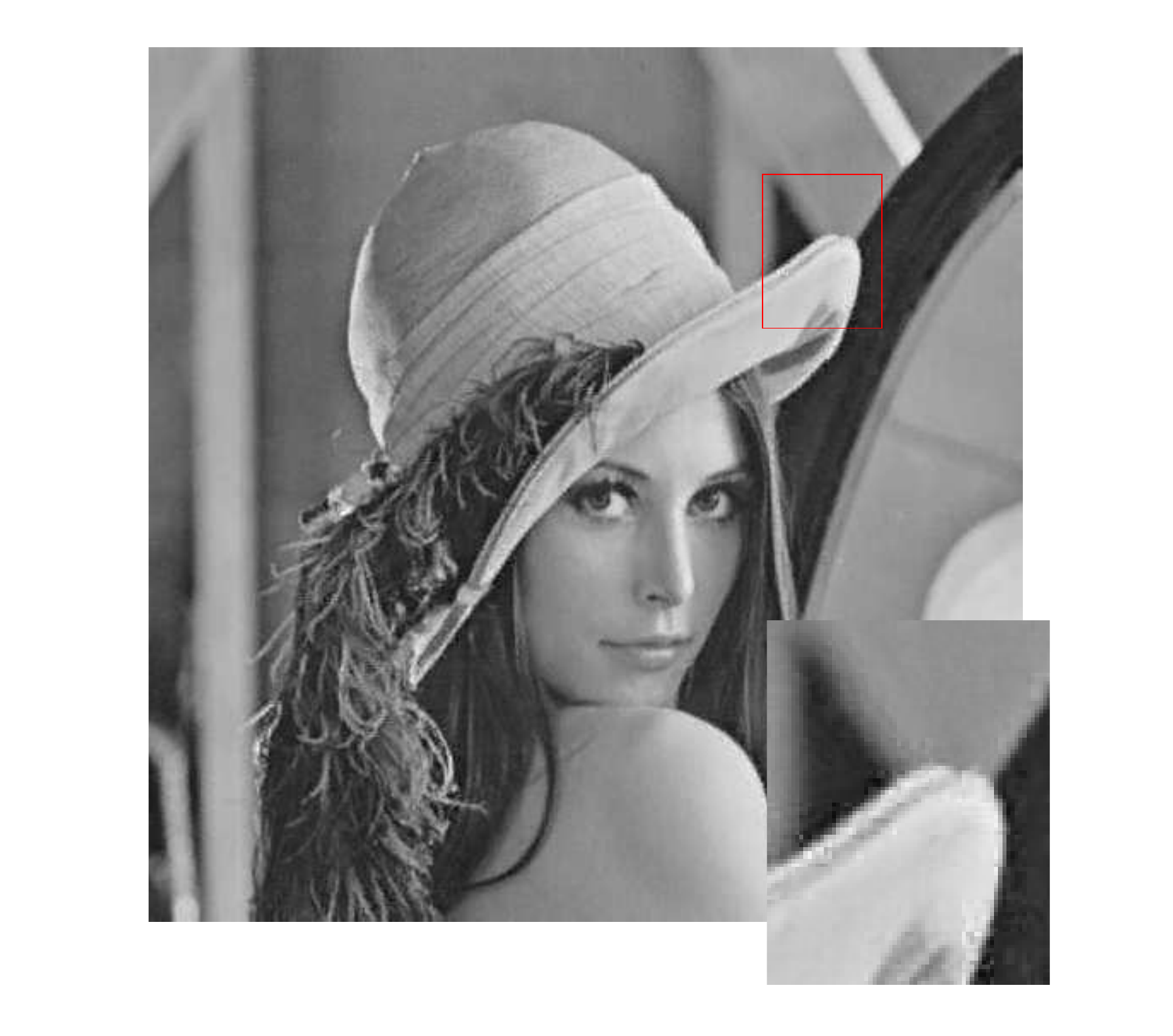}
	}
	\subfigure[$CS_{Alg3}$]{\includegraphics[width=0.23\textwidth]{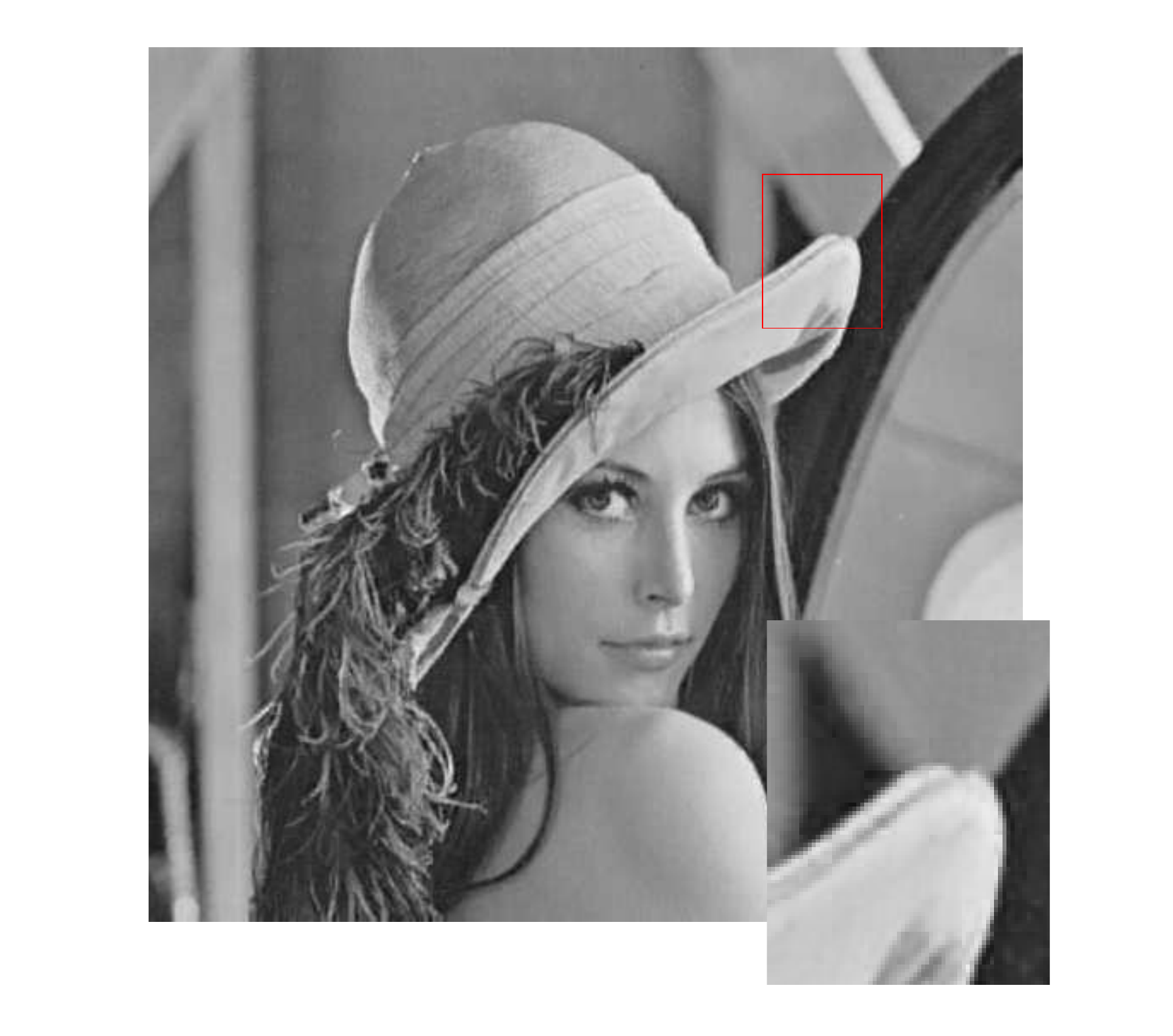}
	}
	\caption{The recovered testing image `Lena'.} \label{Lena:reconst}
\end{figure}

\begin{figure}[!htb]
	\centering
	\subfigure[$CS_{S-DCS}-small$]{\includegraphics[width=0.23\textwidth]{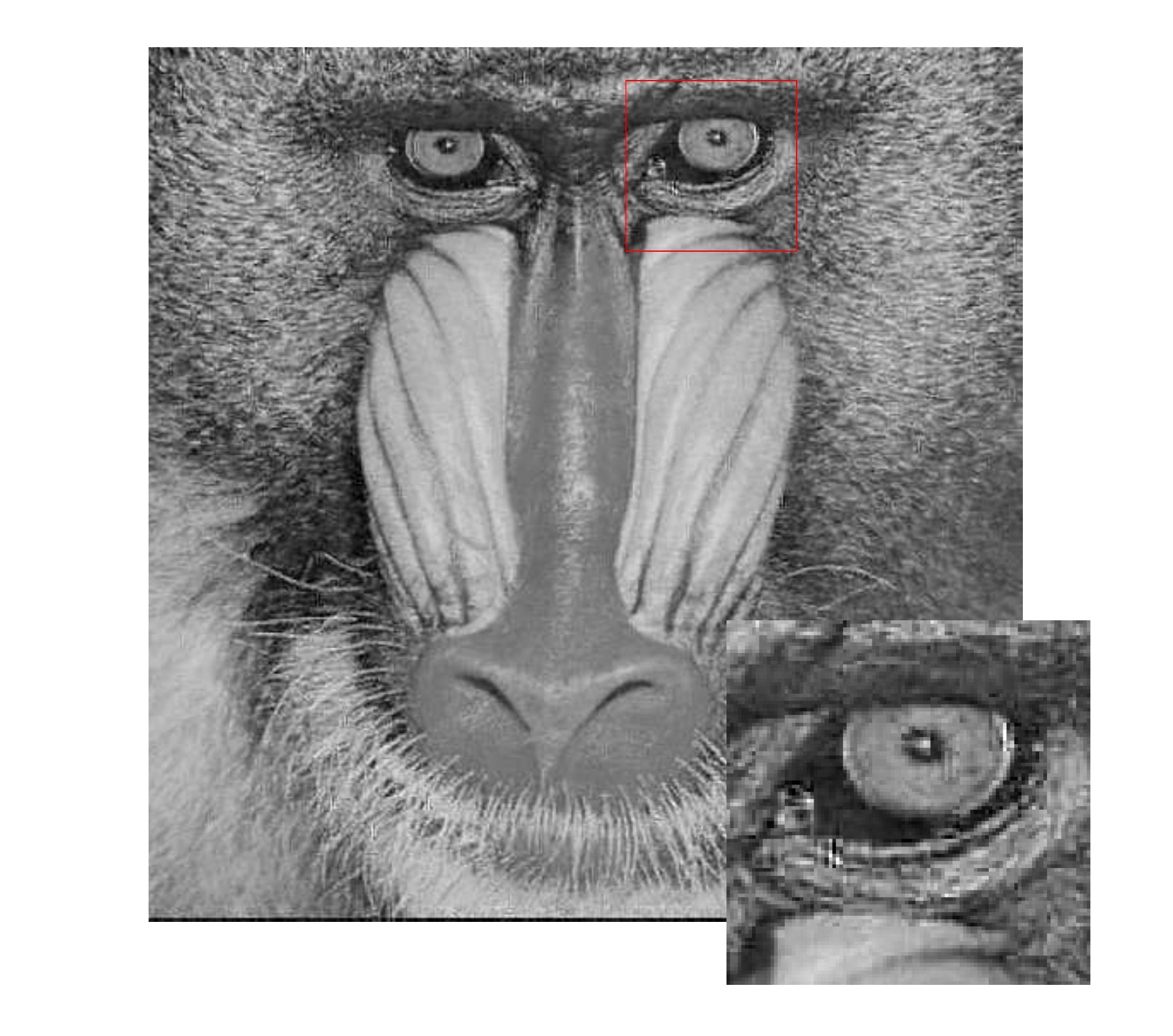}
	}
	\subfigure[$CS_{S-DCS}$]{\includegraphics[width=0.23\textwidth]{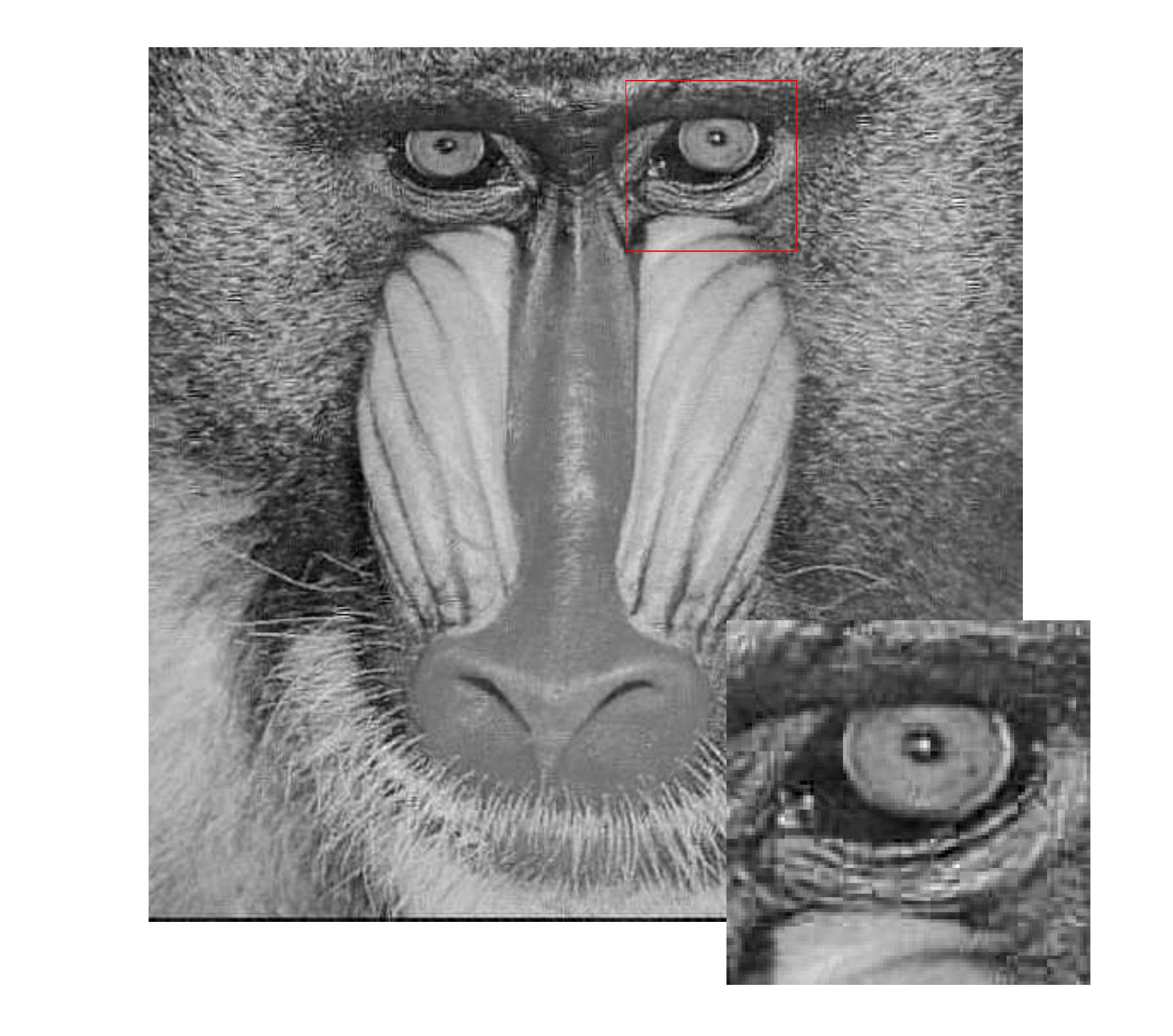}
	}
	\subfigure[$CS_{BL}$]{\includegraphics[width=0.23\textwidth]{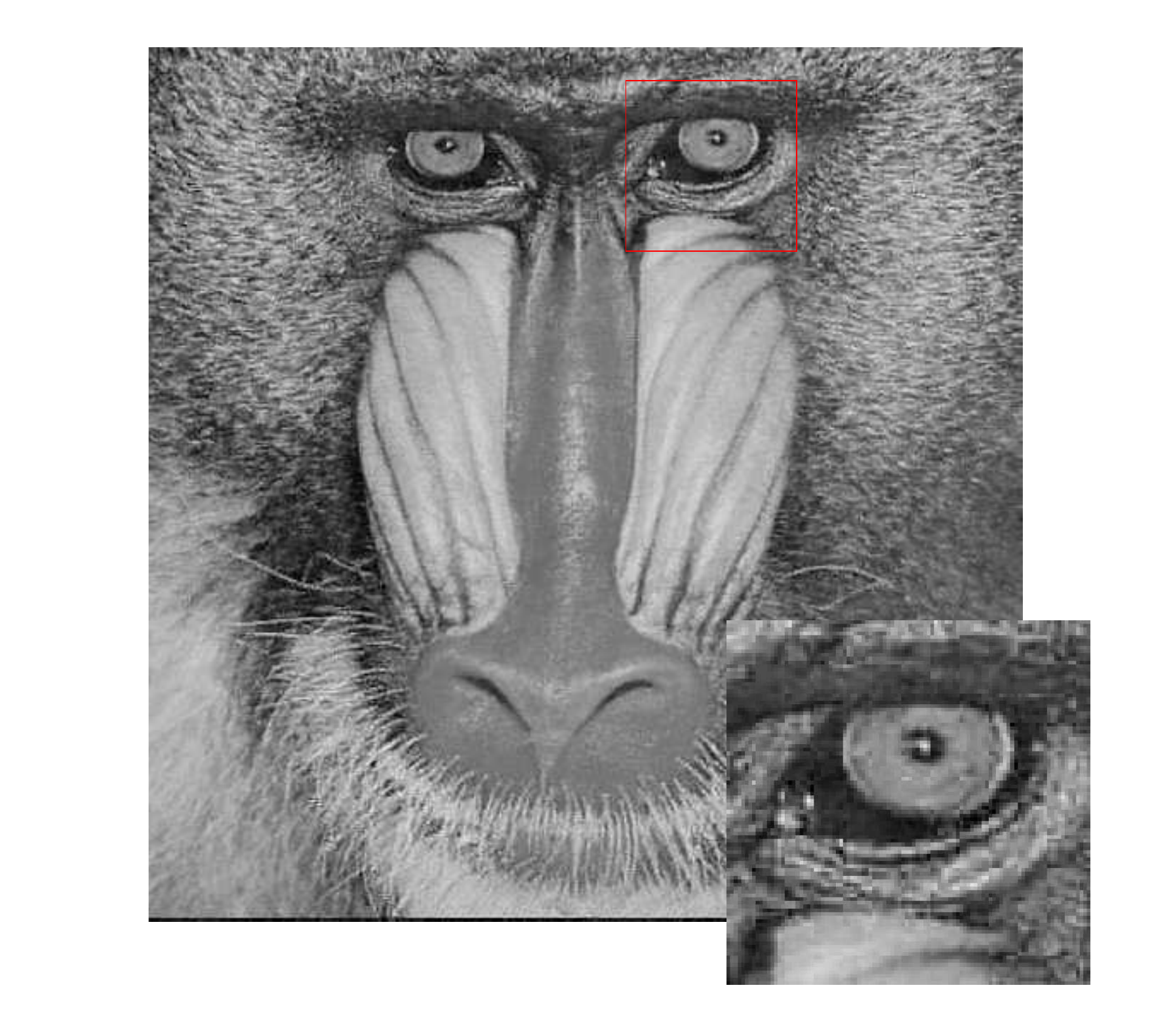}
	}
	\subfigure[$CS_{Alg3}$]{\includegraphics[width=0.23\textwidth]{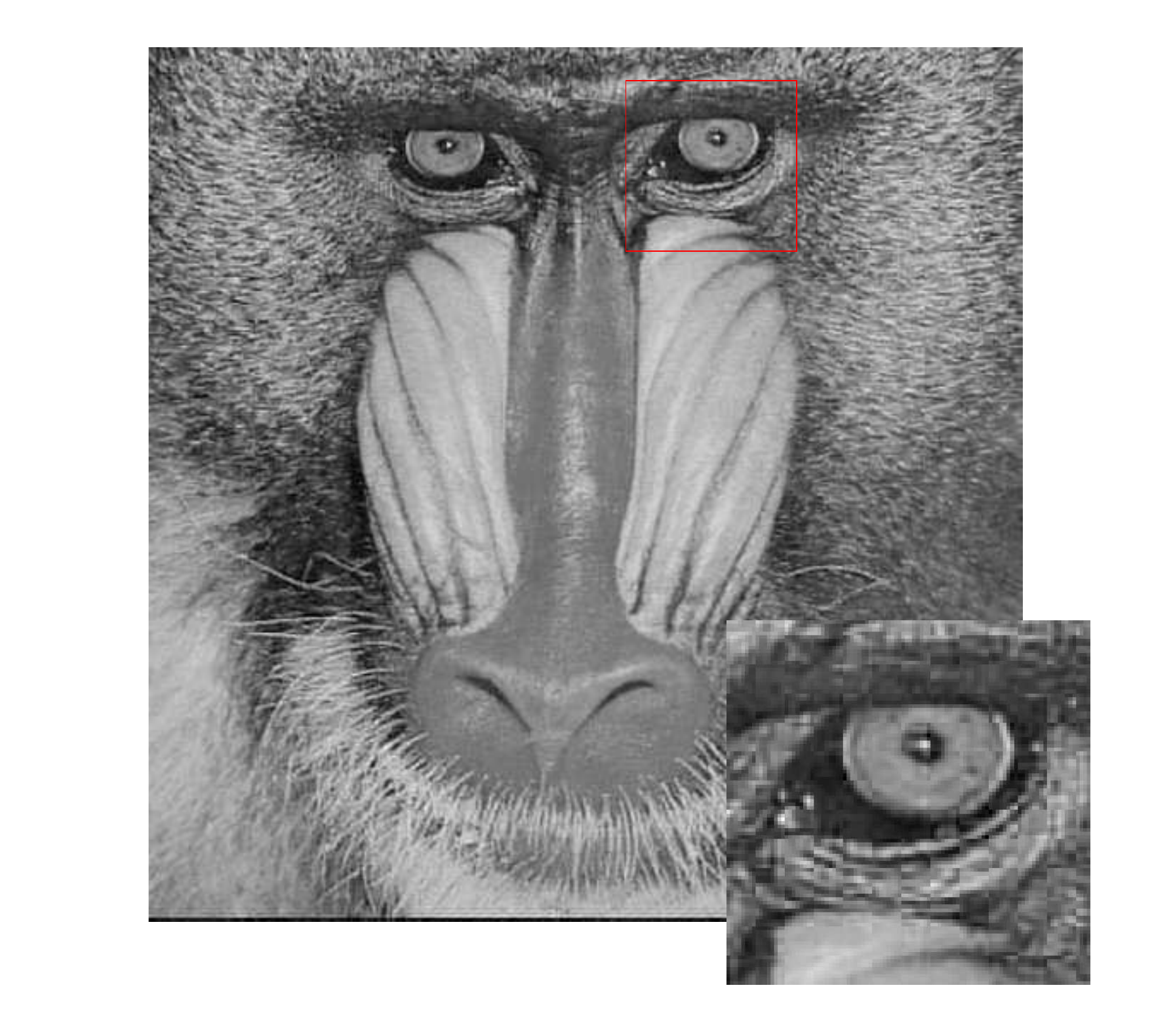}
	}
	\caption{The recovered testing image `Mandrill'.}\label{Mandril:reconst}
\end{figure}

Now, we come to experimentally check the convergence of our proposed algorithm. In this experiment, we run another sufficient large iterations on dictionary updating after running {\bf Algorithm \ref{Alg_joint_projection_dictionary}} to see whether the dictionary can converge. The testing error versus iteration on testing data is shown in Fig. \ref{Convergence:experiment}.\footnote{Although we train our $\bm\Phi$ and $\bm\Psi$ on training data, we only care about the performance on the testing data. So we prefer to see the value of objective function on testing data. Note that the iteration here refers to the total of $Iter_{dic}$ as shown in {\bf Algorithm \ref{Alg_dictionary_involve_projection}}.}  Clearly, even if the testing error is not monotonically decreasing, it is asymptotically decreasing which meets the property of our online algorithm ({\bf Proposition \ref{prop:decrease_func}}) because we randomly sample part of the training data to update the dictionary at each iteration. We can also observe that the recovery accuracy in terms of $\varrho_{psnr}$ on testing data is also increasing along the number of iterations growing. {As seen from the sub-figure in Fig. \ref{Convergence:experiment}, the recovered PSNR increases dramatically after the $1000$-th iteration, in which we update the sensing matrix again. Moreover, we see that the PSNR still increases as the iteration goes, which demonstrates the significance of our algorithm to simultaneously optimize the sensing matrix and the dictionary.}  {We also display the difference of the dictionary between each iteration in Fig. \ref{Convergence:Diff_dic}. As observed from Fig. \ref{Convergence:Diff_dic:a}, there exist many oscillations which are caused by the fact that we update the dictionary through the stochastic method which only utilizes a part of training data in each iteration. If we check Fig. \ref{Convergence:Diff_dic:b}, the envelop of Fig. \ref{Convergence:Diff_dic:a}, we see it is convergent and coincides with the {\bf Proposition \ref{prop:dic_dif}} that the stationary point can be attained.} Note that all of the observations meet our previous statements in Section \ref{S_3} regarding the convergence analysis of the proposed {\bf Algorithm \ref{Alg_joint_projection_dictionary}}. However, the whole investigation of the convergence analysis for {\bf Algorithm \ref{Alg_joint_projection_dictionary}} is out of the scope in this paper and belongs to future work.

\begin{figure}[!htb]
	\centering
	\includegraphics[width=0.5\textwidth]{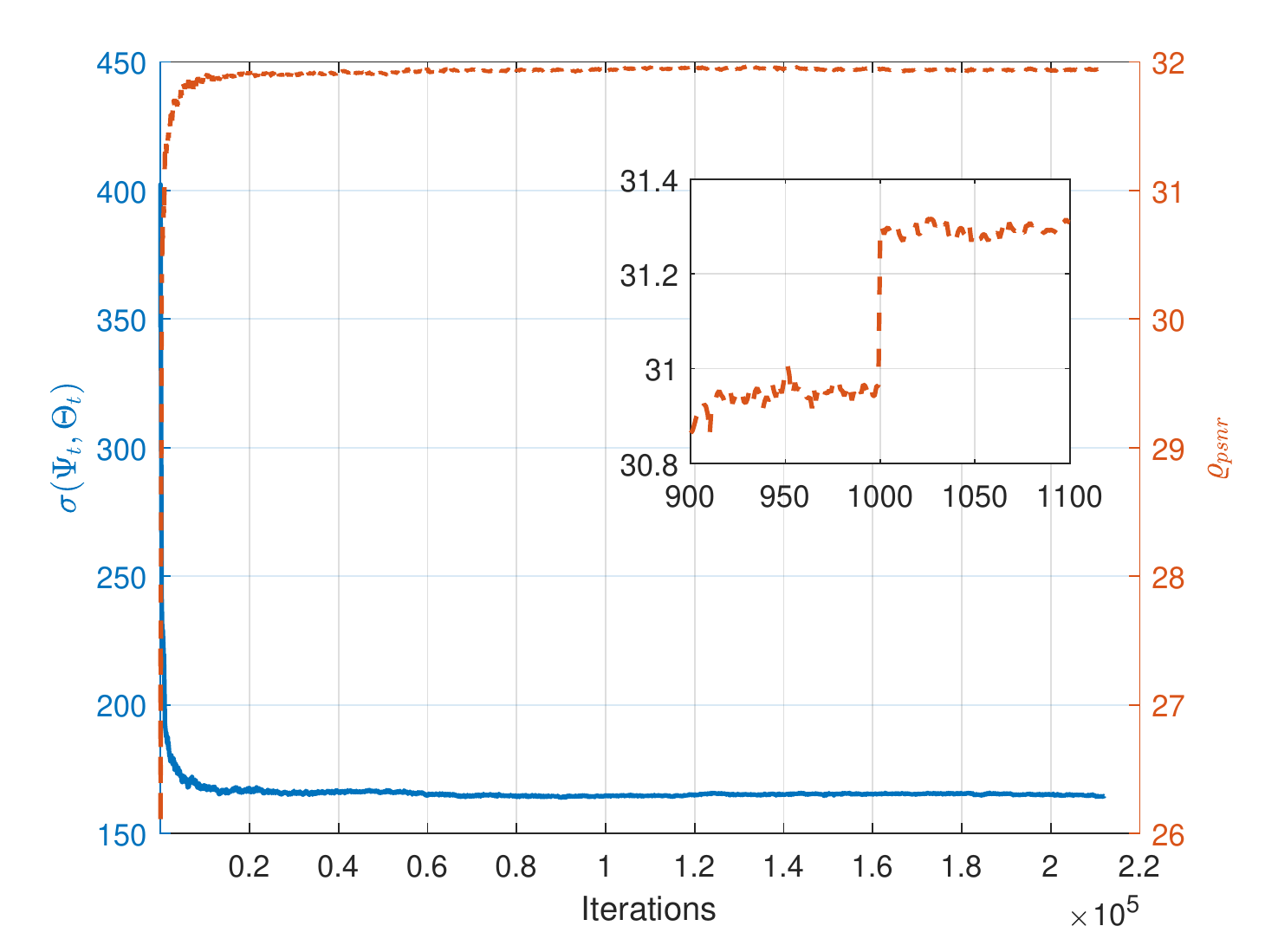}\\
	\caption{Objective value $\sigma(\bm\Psi_t,\bm\Theta_t)$ and $\sigma_{psnr}$ versus iteration on testing data through {\bf Algorithm $3$}.}\label{Convergence:experiment} %{\color{red} test on both of the different initial dictionary.}
\end{figure}
%\begin{figure}[!htb]
%	\centering
%	% Requires \usepackage{graphicx}
%	\includegraphics[width=0.48\textwidth]{fig/Dif_Dic.pdf}\\
%	\caption{The change of the dictionary in each iteration.}\label{Convergence:Diff_dic} %{\color{red} test on both of the different initial dictionary.}
%\end{figure}

\begin{figure}[!htb]
	\centering
	\subfigure[The difference between each dictionary versus iteration .]{\includegraphics[width=0.4\textwidth]{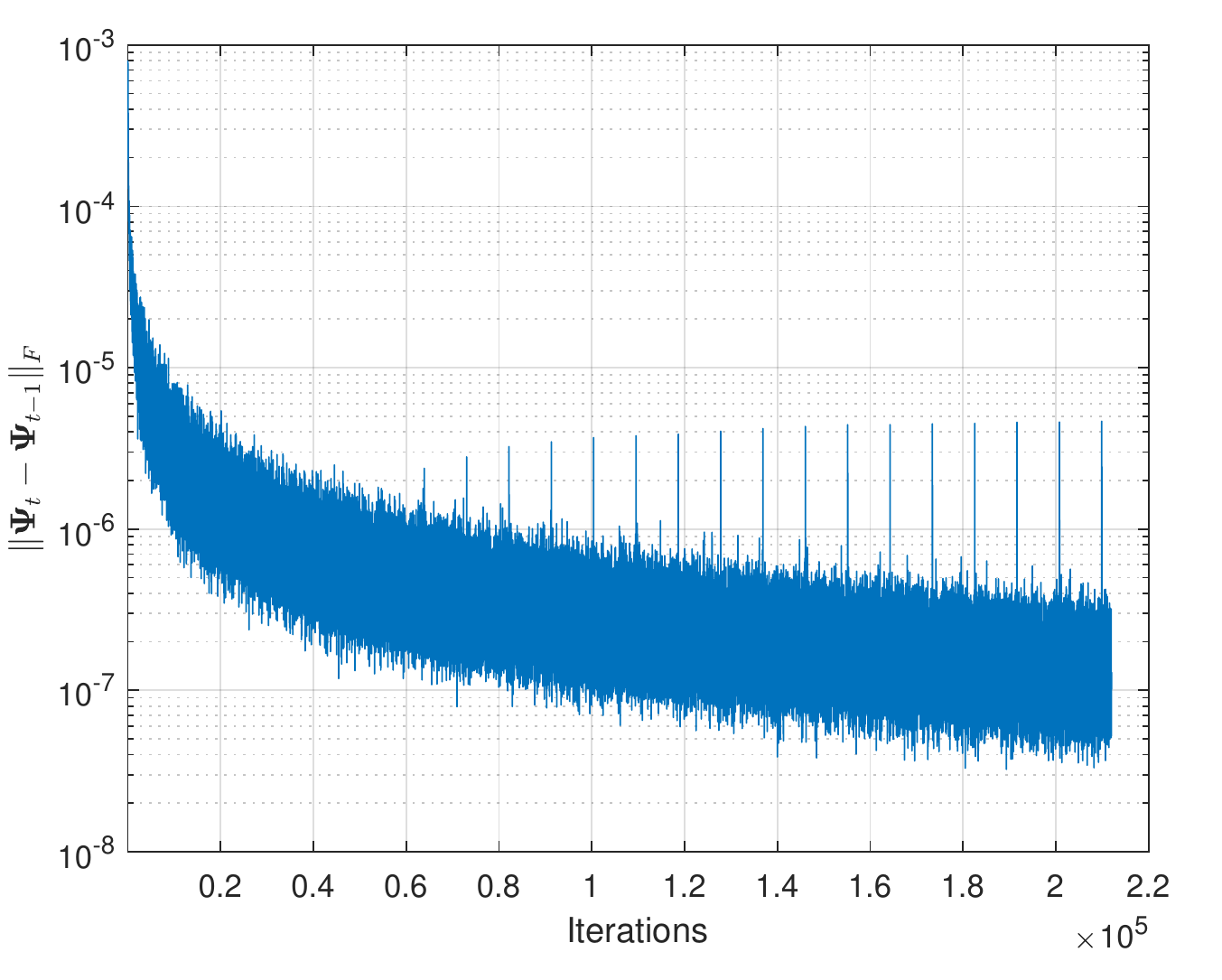}\label{Convergence:Diff_dic:a}}
	\subfigure[The envelop of the difference between each dictionary versus iteration.]{\includegraphics[width=0.4\textwidth]{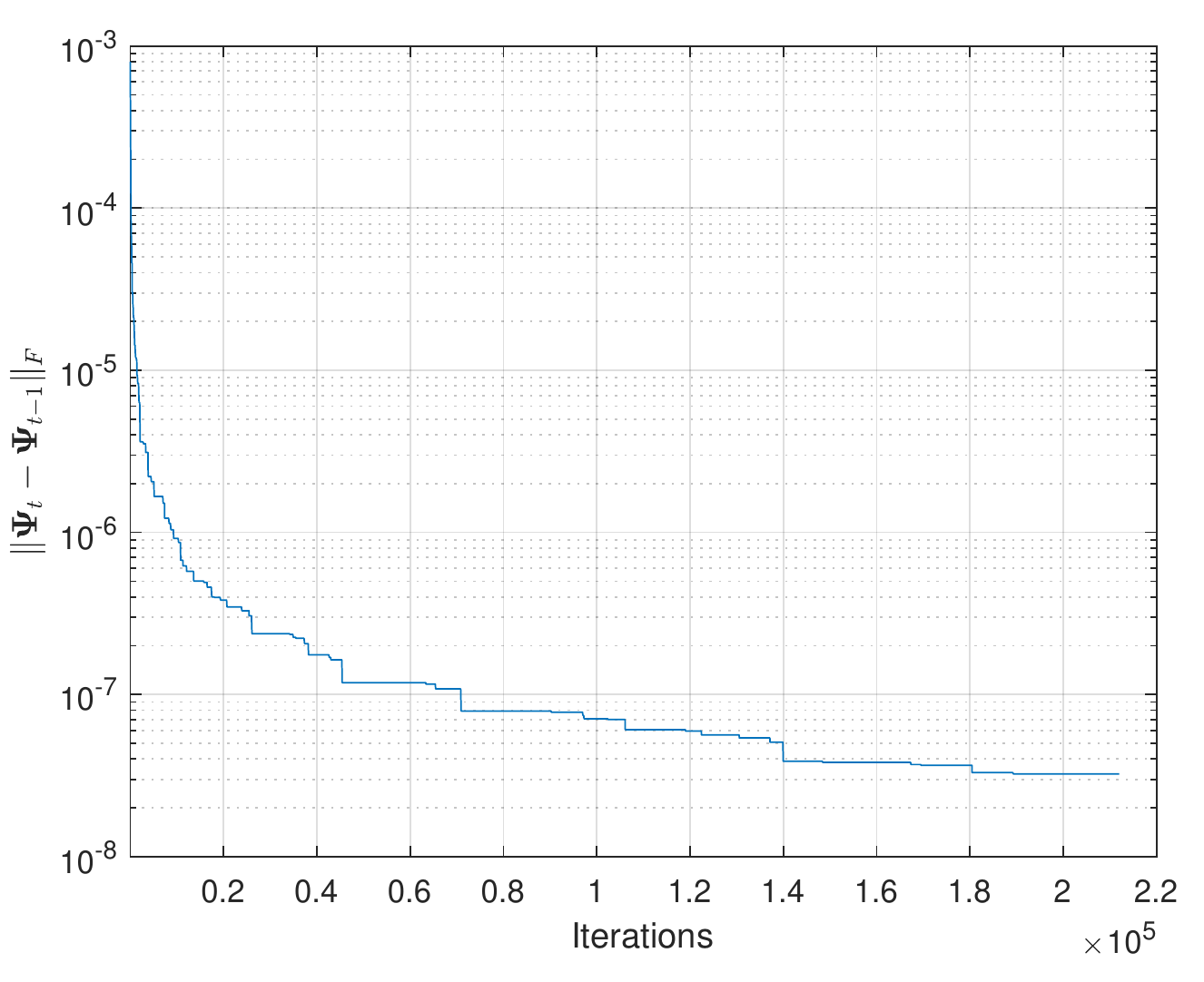}\label{Convergence:Diff_dic:b}}
	\caption{The change of each dictionary in each iteration.}\label{Convergence:Diff_dic} %{\color{red} test on both of the different initial dictionary.}
\end{figure}
%Moreover, we also test the results on natural images. The author can find the results in the supplemental material.

\section{Conclusion}\label{S_5}
In this paper, an efficient algorithm for jointly learning the SMSD on a large dataset is proposed. The proposed algorithm optimizes the sensing matrix with a closed-form solution and learns a sparsifying dictionary with a stochastic method on a large training dataset. Our experiment results show that training the SMSD on a large dataset yildes a better performance and the proposed method which considers the efficiency and effectiveness simultaneously is a suitable choice for such a task.

One of the possible directions for future research is to develop an accelerated algorithm to make the proposed method more efficient. Involving the Sequential Subspace Optimization (SESOP) in the algorithm may belong to one of the possible methods to realize the accelerated purpose \cite{RHGZ16}.

\section*{Acknowledgment}
This research is supported in part by ERC Grant agreement no. 320649, and in part by the Intel Collaborative Research Institute for Computational Intelligence (ICRI-CI). The code
in this paper to represent the experiments can be downloaded through the link {https://github.com/happyhongt/}.


\end{document}